%% file: main.tex
\documentclass[10pt,journal,compsoc]{IEEEtran}
\usepackage[nocompress]{cite}
\usepackage{url}            
\usepackage{booktabs}       
\usepackage{amsfonts}       
\usepackage{nicefrac}       
\usepackage{microtype}      
\usepackage{xcolor}         
\usepackage{epsfig}
\usepackage{graphicx}
\usepackage{amsmath}
\usepackage{amssymb}
\usepackage{multirow}
\usepackage{color, colortbl}
\usepackage{subcaption}
\usepackage{algorithm,algpseudocode}
\usepackage{pifont}

\usepackage{xspace}

\usepackage{comment}
\usepackage{bm}
\usepackage{wrapfig}
\usepackage[pagebackref=true,breaklinks=true,colorlinks,bookmarks=false]{hyperref}
\usepackage{algorithm}
\usepackage{algpseudocode}

\newcommand{\keypoint}[1]{\noindent\textbf{#1}\quad}
\makeatletter
\DeclareRobustCommand\onedot{\futurelet\@let@token\@onedot}
\def\@onedot{\ifx\@let@token.\else.\null\fi\xspace}

\makeatother

\definecolor{blue1}{rgb}{0.8,0.9,1}   
\definecolor{blue2}{rgb}{0.6,0.8,1}
\definecolor{blue3}{rgb}{0.4,0.6,1}
\definecolor{blue4}{rgb}{0.2,0.4,1}
\definecolor{blue5}{rgb}{0.0,0.2,1}   

\newtheorem{lemma}{Lemma}
\newtheorem{proof}{Proof}[section]

\begin{document}
\title{From Misclassifications to Outliers: \\Joint Reliability Assessment in Classification}

\author{Yang~Li, Youyang~Sha, Yinzhi~Wang, Timothy~Hospedales, Xi~Shen, Shell~Xu~Hu, Xuanlong~Yu \IEEEcompsocitemizethanks{ \IEEEcompsocthanksitem Yang~Li, Youyang~Sha, Yinzhi~Wang, Xi~Shen, and Xuanlong~Yu are with Intellindust AI Lab, Shenzhen, 518000, China. 
\protect\\ {\sloppy E-mail: \{liyang, shayouyang, yingzhiwang, shenxi, yuxuanlong\}@intellindust.com } \IEEEcompsocthanksitem Shell~Xu~Hu and Timothy~Hospedales are with Samsung AI Center, Cambridge, CB1 2GA, United Kingdom. \protect\\ E-mail: \{shell.hu, t.hospedales\}@samsung.com} }

\markboth{Submitted to IEEE Transactions on Pattern Analysis and Machine Intelligence}%
{Li \MakeLowercase{\textit{et al.}}: 
From Misclassifications to Outliers: Joint Reliability Assessment in Classification}

\input{sections/0_abstract}
\input{sections/1_introduction}
\input{sections/2_background}

\input{sections/3_method_1}

\input{sections/4_method_2}
\input{sections/5_exp}

\input{sections/7_futurework}

\input{sections/8_conclusion}
\input{sections/9_acknowledge}


\bibliographystyle{ieeetr}
\bibliography{citation}

\clearpage
\input{supp}


\end{document}

%% file: sections/0_abstract.tex
\IEEEtitleabstractindextext{%
\begin{abstract}
Building reliable classifiers is a fundamental challenge for deploying machine learning in real-world applications. A reliable system should not only detect out-of-distribution (OOD) inputs but also anticipate in-distribution (ID) errors by assigning low confidence to potentially misclassified samples. Yet, most prior work treats OOD detection and failure prediction as separated problems, overlooking their closed connection. We argue that reliability requires evaluating them jointly. To this end, we propose a unified evaluation framework that integrates OOD detection and failure prediction, quantified by our new metrics DS-F1 and DS-AURC, where DS denotes double scoring functions. Experiments on the OpenOOD benchmark show that double scoring functions yield classifiers that are substantially more reliable than traditional single scoring approaches. Our analysis further reveals that OOD-based approaches provide notable gains under simple or far-OOD shifts, but only marginal benefits under more challenging near-OOD conditions. Beyond evaluation, we extend the reliable classifier SURE and introduce SURE+, a new approach that significantly improves reliability across diverse scenarios. Together, our framework, metrics, and method establish a new benchmark for trustworthy classification and offer practical guidance for deploying robust models in real-world settings. The source code is publicly available at \url{https://github.com/Intellindust-AI-Lab/SUREPlus}.
\end{abstract}

}

\maketitle
\IEEEdisplaynontitleabstractindextext
\IEEEpeerreviewmaketitle

%% file: sections/1_introduction.tex
\IEEEraisesectionheading{\section{Introduction}}
\label{sec:introduction}
\IEEEPARstart{D}{}eploying machine learning classifiers in safety-critical domains such as fire and smoke detection demands model robustness beyond high benchmark accuracy. In real-world environments, a reliable system must not only detect actual fire and smoke, but also avoid false alarms caused by visually similar yet harmless phenomena (e.g., fog, steam, or unusual lighting). Equally important, the system should be able to recognize its own uncertainty and refrain from making overconfident misclassifications. Failures in either direction--failing to detect a true fire or generating frequent false alarms--can lead to severe consequences, ranging from safety risks to a loss of user trust. Ensuring reliability under these conditions is therefore a fundamental and urgent challenge.

The research community has actively explored two directions that are highly related to reliability. The first is out-of-distribution (OOD) detection~\cite{yang2022openood,zhang2023openood,yang2024generalized}, which aims to identify inputs that deviate from the training distribution and should not be trusted. The second is failure prediction~\cite{corbiere2019addressing,zhu2022rethinking}, which estimates whether a classifier’s prediction on an in-distribution (ID) sample is correct.
Both directions have established benchmarks, specialized algorithms, and thriving research communities. However, most existing works study these two aspects separately, treating them as independent problems.

\input{figure1}

In real-world applications, a classifier must handle both in-distribution (ID) and out-of-distribution (OOD) inputs. This creates the challenge of jointly addressing OOD detection and failure prediction, so that model reliability can be assessed in a unified way. Suppose the model produces two scores for each input: \textit{i) an OOD detection score $s_{\rm OOD}$, and ii) an in-distribution confidence score $s_{\rm ID}$.} Together, these scores form a binary decision system:
\begin{itemize}
    \item First, we ask: Is the input ID or OOD? If $s_{\rm OOD}$ exceeds a threshold $\tau_{\rm OOD}$, the input is accepted as ID.
    \item Next, if the input is ID, we ask: Can we trust the predicted label? This is determined by applying another threshold $\tau_{\rm ID}$ on $s_{\rm ID}$.
\end{itemize}

Addressing both dimensions at once is challenging. A related direction is selective classification~\cite{geifman2017selective}, which evaluates both ID accuracy and OOD rejection within a unified framework~\cite{xia2022augmenting}. However, most selective classification methods rely on a \textbf{single scoring function} with one decision threshold to accept or reject predictions. This approach does not fully take advantage of the specialized progress made in the OOD detection and failure prediction communities.

The challenge of evaluation has already been emphasized in prior studies. For instance, the OpenOOD benchmark~\cite{yang2022openood,zhang2023openood} shows that no single method consistently excels across all metrics, complicating the choice of a training strategy or post-hoc approach for practical use. While this issue arises in the context of OOD detection alone, it becomes even more complex when ID accuracy is also considered, since OOD detection and failure prediction are typically evaluated in isolation. Such fragmented evaluations lead to inconsistencies and hinder reliable deployment. What is needed instead is a unified metric that more intuitively captures performance on the joint task.

To illustrate the tension between ID reliability and OOD performance, we train two ResNet-18~\cite{he2016deep} classifiers on CIFAR-100~\cite{krizhevsky2009learning}: one using standard cross-entropy loss (CE) and the other with CutMix~\cite{yun2019cutmix}. We then evaluate both models on CIFAR-100 (ID) and MNIST (OOD)~\cite{lecun1998mnist}. As shown in the first three columns of Figure~\ref{fig:ce_regmixup}, CutMix slightly improves ID accuracy and boosts OOD detection performance according to Area Under the Receiver Operating Characteristic Curve (AUROC) metric and using KLM~\cite{basart2022scaling} as the scoring function for $s_\text{OOD}$,
but it decreases reliability on ID samples according to Area Under the Risk-Coverage curve (AURC) metric on the ID test set when taking MSP~\cite{hendrycks2016baseline} for $s_\text{ID}$.
This highlights a common dilemma: both ID and OOD performance are crucial for building a reliable classifier in real-world settings. Focusing on only ID accuracy or OOD detection can be misleading, making it difficult to determine which model is truly dependable.


To address this challenge, we propose a unified framework that treats OOD detection and failure prediction as complementary aspects of classifier reliability. Rather than evaluating these tasks separately, we assess them jointly using two scoring functions and their corresponding thresholds: $\tau_{\rm OOD}$ for detecting OOD samples and $\tau_{\rm ID}$ for assessing ID classification confidence. This double-scoring approach divides all samples into four categories: \textit{True Accept}, \textit{True Reject}, \textit{False Accept}, and \textit{False Reject}. For example, True Accept refers to ID samples that pass both thresholds and are correctly classified by the model. Using these categories, we extend standard metrics such as AURC and F1 to account for both ID and OOD data, resulting in DS-AURC and DS-F1 (DS = Double Scoring). In particular, DS-F1 searches for the best F1 score across the double-scoring surface, as illustrated in Figure~\ref{fig:surface}. These metrics give a more accurate measure of a classifier’s reliability and help identify truly robust models, avoiding misleading conclusions that arise when considering only one aspect in isolation. As shown in the last two columns of Figure~\ref{fig:ce_regmixup}, when considering simultaneously KLM and MSP as $s_\text{OOD}$ and $s_\text{ID}$, our DS-F1 and DS-AURC metrics show that CE achieves a higher DS-F1 and a lower DS-AURC, indicating better overall performance when evaluating ID and OOD performance together.

We show that for double scoring functions, which jointly consider the OOD score ($s_{\rm OOD}$) and the ID confidence ($s_{\rm ID}$), the evaluation metrics naturally reduce to their single-score counterparts when only one score is used. Concretely, DS-AURC is lower bounded by the standard AURC, while DS-F1 is guaranteed to be no worse than the best F1 obtained from a single scoring function. This ensures that double scoring at least matches single-score methods while providing a more faithful measure of classifier reliability.
To validate this, we conduct extensive experiments on OpenOOD~\cite{yang2022openood}, showing that double scoring consistently produces more robust and reliable classifiers. 

Interestingly, post-hoc OOD scores ($s_{\rm OOD}$) complement MSP ($s_{\rm ID}$) effectively under far-OOD shifts, yet provide little additional benefit under near-OOD conditions, highlighting the limits of relying solely on advanced post-hoc OOD detection approaches.

Beyond the unified evaluation metric, we extend the reliable classifier SURE~\cite{li2024sure}, originally focused on failure prediction, to handle both ID and OOD scenarios. Building on this, we introduce SURE+, a streamlined and more powerful version that incorporates recent advances in both OOD detection and failure prediction. With this unified and simplified design, SURE+ achieves significantly higher reliability when evaluated on both ID and OOD samples.

Our key contributions are as follows:
\begin{itemize}
\item We reveal that \emph{OOD detection} and \emph{failure prediction}, though often studied in isolation, are inherently complementary aspects of classifier reliability. Evaluating them separately can lead to misleading conclusions, whereas a unified perspective provides a more faithful reflection of real-world deployment needs.
\item To this end, we introduce two complementary metrics, \textbf{DS-F1} and \textbf{DS-AURC}, which jointly measure a classifier’s ability to detect OOD inputs and anticipate its own misclassifications. These metrics offer a principled and comprehensive way to evaluate model reliability.
\item Extensive experiments on the OpenOOD benchmark demonstrate that our unified framework consistently identifies classifiers that are substantially more robust and trustworthy than those optimized for OOD detection or failure prediction alone.
\item Beyond evaluation, we propose \textbf{SURE+}, an improved reliable classifier that builds on SURE by integrating recent advances in both OOD detection and failure prediction. Experiments show that SURE+ achieves state-of-the-art reliability across diverse scenarios.
\end{itemize}

Code will be released upon publication.

%% file: figure1.tex
\begin{figure*}[t]
\centering
\begin{subfigure}[b]{0.48\linewidth}
    \centering
    \includegraphics[width=\linewidth]{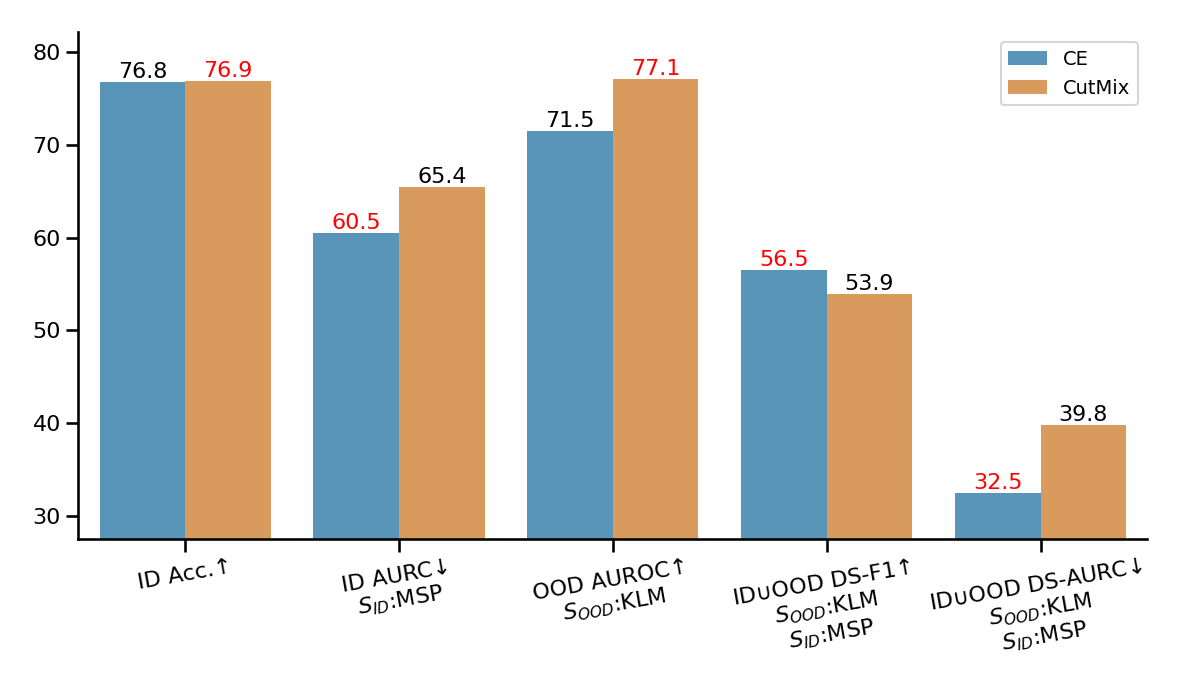}
    \caption{CE vs. CutMix}
    \label{fig:ce_regmixup}
\end{subfigure}
\hfill
\begin{subfigure}[b]{0.48\linewidth}
    \centering
    \includegraphics[width=\linewidth]{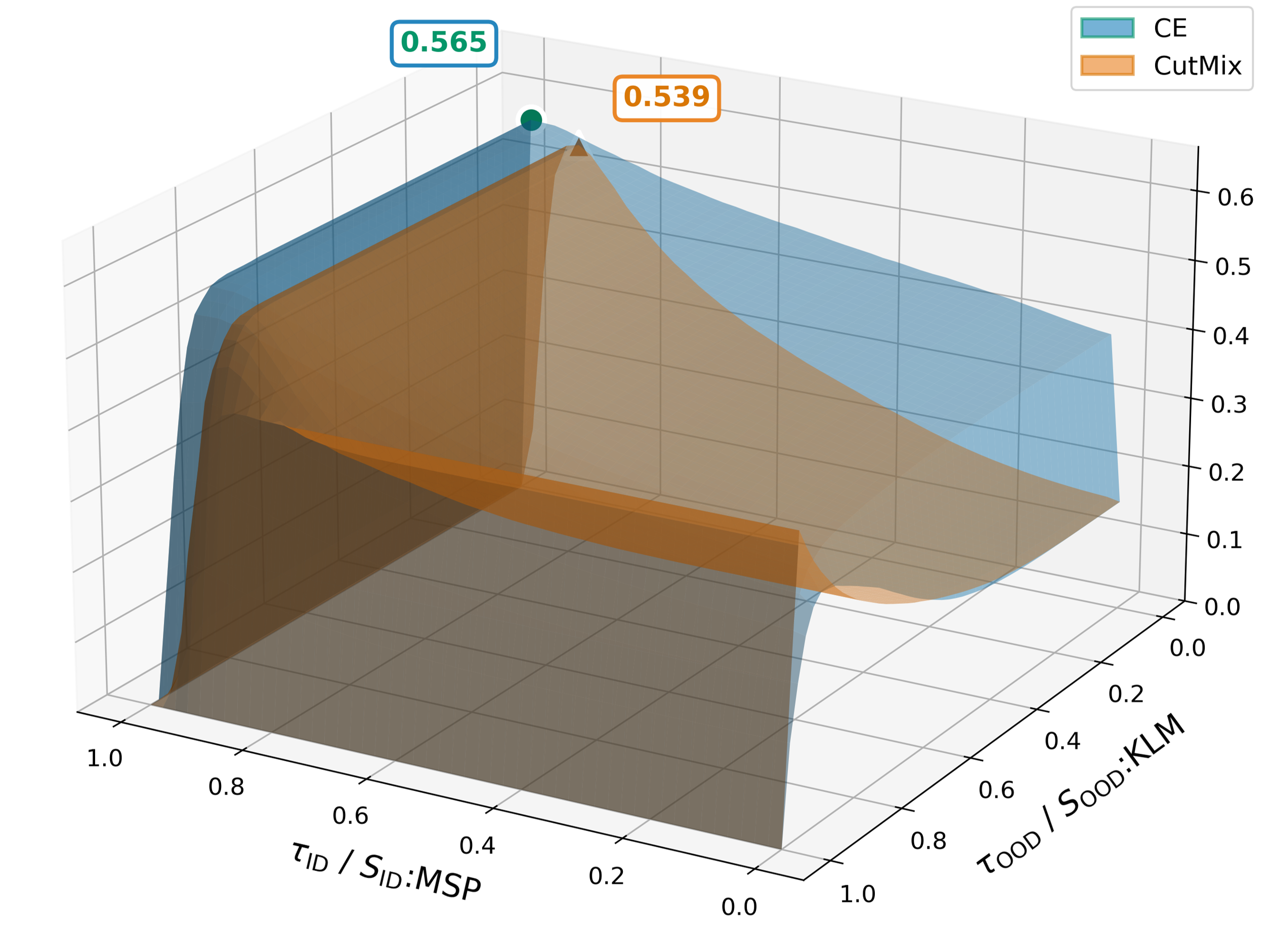}
    \caption{F1 vs. threshold of ID $\tau_{\rm ID}$ and OOD $\tau_{\rm OOD}$}
    \label{fig:surface}
\end{subfigure}
\caption{\textbf{Joint evaluation of ID and OOD.} a) Two ResNet-18 models were trained with cross-entropy {(CE)} and CutMix separately on CIFAR-100. For ID metrics (the first two columns, on CIFAR-100), CE outperforms CutMix; while for OOD detection (the third column, on MNIST), CutMix is better. Given this dilemma, the previous evaluation systems would struggle to give a sensible ranking between these two methods. Our new  metrics (DS-F1 and DS-AURC, the last two columns) confirm that CE works better than CutMix in a joint ID and OOD evaluation, which has an intuitive explaination in b) that the F1 score surface of CE is above the surface of CutMix for most $(\tau_{\rm ID}, \tau_{\rm OOD})$ pairs, including their best F1 scores.}
\label{fig:reliability_concept}
\end{figure*}

%% file: sections/2_background.tex
\section{Related work}
\subsection{Out-of-distribution detection}
Out-of-distribution (OOD) detection has been widely studied for enhancing the robustness and reliability of models. Existing methods can be broadly divided into post-hoc and training-based methods. Post-hoc methods, such as 
MSP~\cite{hendrycks2016baseline}, Energy~\cite{liu2020energy}, GradNorm~\cite{huang2021importance}, aim to leverage signals already contained in a trained model to distinguish OOD from ID inputs, whether from the model's logits, intermediate features, or their gradients. The training-based methods explicitly modify the learning process, for example, by incorporating auxiliary outliers~\cite{hendrycks2018deep,zhu2023openmix}, masking ID inputs~\cite{li2023rethinking}, or applying regularization losses~\cite{ming2022exploit}, to make models more sensitive to OOD data. Despite their differences, these approaches share the common goal of providing a reliable scoring function, which, combined with a threshold, filters out OOD inputs and prevents the system from mistakenly trusting predictions on them.

\subsection{Failure prediction and selective classification}
Beyond detecting distributional novelty, a key challenge is to predict whether a model’s output is correct. Post-hoc scoring functions introduced for OOD detection can also be applied to failure prediction, e.g. MSP~\cite{hendrycks2016baseline}, while some works exploit intermediate representations of a trained model and train auxiliary modules for estimating prediction errors~\cite{corbiere2019addressing,luo2021learning,yu2021slurp,shen2023post}. SIRC~\cite{xia2022augmenting} adopts a selective classification perspective~\cite{kim2023unified,geifman2017selective}, combining multiple post-hoc scores to filter unreliable predictions and better separate system failures from reliable outputs. However, its evaluation still treats ID and OOD data separately, despite the method's aim to address both. Evaluating a classification system under joint ID and OOD inputs is more realistic but under-explored, as current metrics largely consider these dimensions independently, motivating the need for unified system-level evaluation.

\subsection{Training strategies for reliable classifiers} 
Post-hoc scoring methods rely on pre-trained models, and their effectiveness is closely tied to the underlying training scheme. Data augmentation strategies such as Mixup~\cite{thulasidasan2019mixup} and RegMixup~\cite{pinto2022using} enhance generalization and OOD detection. Optimization techniques like SAM~\cite{foret2020sharpness} and F-SAM~\cite{li2024friendly} encourage flat minima, while weight averaging approaches like SWA~\cite{izmailov2018averaging} and SWAG~\cite{maddox2019simple} improve stability and calibration. More recent pipelines such as FMFP~\cite{zhu2023revisiting} and SURE~\cite{li2024sure} integrate these ideas, achieving strong performance in failure prediction and robustness. Yet, the behavior of SURE under OOD exposure remains underexplored. In this work, we introduce SURE+, an enhanced version designed to improve reliability and robustness in realistic scenarios where both ID and OOD samples are present.

\subsection{Existing evaluation metrics for classification system}
Metrics for failure prediction and OOD detection are typically based on the selective classification principle, where a scoring function and threshold decide whether to trust a prediction. They can be grouped into single-threshold metrics, such as AUROC, AUPR~\cite{hendrycks2016baseline,yang2022openood,zhang2023openood}, FPR@95, and F1 score, which evaluate performance at a fixed threshold, and multi-threshold metrics, such as AURC~\cite{geifman2018biasreduced} and recently proposed AUGRC~\cite{traub2024overcoming}, which aggregate performance across thresholds to assess risk over coverage.

Recognizing the limitations of isolated evaluations, prior works~\cite{jaeger2022call,bernhardt2022failure,kim2023unified,narasimhan2024plugin} have advocated for the necessity of testing classification systems in mixed ID and OOD settings. The motivation of these works aligns closely with this work. These studies argue that a truly reliable system must simultaneously handle both misclassifications of known classes and the presence of unknown outliers~\cite{jaeger2022call,kim2023unified}. For instance, Jaeger et al.~\cite{kim2023unified} and Bernhardt et al.~\cite{bernhardt2022failure} systematically reveal that improvements in OOD detection do not necessarily translate to better failure detection, calling for a holistic evaluation of all potential risk sources.

However, these works~\cite{jaeger2022call,bernhardt2022failure,kim2023unified,narasimhan2024plugin} still implicitly assume that a single scoring rule suffices for building a reliable system. This technical constraint naturally leads to methodologies focused on merely strengthening or modifying a single post-hoc signal to find a compromise between ID and OOD tasks~\cite{xia2022augmenting,narasimhan2024plugin}. In practice, the objective of accepting correct ID samples while rejecting both ID failures and OOD samples on a joint evaluation set is difficult to achieve with a single scalar. This calls for a more comprehensive evaluation paradigm where two scoring functions with corresponding thresholds are more effective than one. To this end, we propose DS-F1 and DS-AURC, which naturally extend selective classification by incorporating double scoring to better capture and evaluate system reliability.

%% file: sections/3_method_1.tex
\section{DS-F1 and DS-AURC: Novel Evaluation Metrics}
We here consider that double scoring functions perform OOD detection and failure prediction, each has its own threshold, and the functions could share the same post-hoc method. This matches more in real-world settings, for instance, one can use MSP score~\cite{hendrycks2016baseline} for both OOD and failure prediction, but can also apply Energy score~\cite{liu2020energy} for OOD detection and MSP score for failure prediction.
This motivates us to extend the existing single-rule framework into a double-rule setting, 
leading to our proposed DS-F1 and DS-AURC metrics. 

Let $\mathcal{D} = \mathcal{D}_{\text{ID}} \cup \mathcal{D}_{\text{OOD}}$ be the evaluation dataset, comprising both ID and OOD samples, with $N$ samples in total. For an input $x$, a model $m$ produces a prediction $m(x)$. We then have two scoring functions based on $m$:
\begin{enumerate}
    \item An \textbf{OOD detection score} $s_\text{OOD}(x) \in \mathbb{R}$, where higher values indicate a higher likelihood of the sample being ID.
    \item A \textbf{failure prediction score} $s_\text{ID}(x) \in \mathbb{R}$, where higher values indicate higher likelihood in the prediction $m(x)$ being correct.
\end{enumerate}
The decision to accept a prediction is based on two corresponding thresholds, $\tau_\text{OOD}$ and $\tau_\text{ID} \in \mathbb{R}$.

\subsection{DS-F1}

The classical F1-score balances precision and recall for a single binary decision rule. We extend this concept to our double-scoring setting to find the optimal joint operating point of the system. The goal is to \textbf{accept correctly classified ID samples} while rejecting both OOD samples and misclassified ID samples.

Specifically, for any threshold pair $(\tau_\text{OOD},\tau_\text{ID})$, the acceptance set is
\begin{equation}
\mathcal{A}(\tau_g,\tau_h) = \{ i : s_\text{OOD}(x_i)\ge \tau_\text{OOD} \wedge s_\text{ID}(x_i)\ge \tau_\text{ID} \}.
\label{eq:setA}
\end{equation}
To compute the F1-score and avoid confusion with standard "true positive / false negative" terminology, we adopt an acceptance-oriented notation. We define:

\keypoint{True Accept (TA):}
ID samples that are accepted and correctly classified.
    
\begin{equation*}
        \text{TA}(\tau_\text{OOD}, \tau_\text{ID}) = 
        \sum_{i \in \mathcal{D}_{\text{ID}}} 
        \mathbb{I}(i \in \mathcal{A}(\tau_\text{OOD}, \tau_\text{ID})) \cdot \mathbb{I}(m(x_i) = y_i)
\end{equation*}
\keypoint{False Accept (FA):}
    Samples that are accepted but not correct, including accepted OOD samples and misclassified ID samples.
    
    \begin{equation*}
\begin{aligned}
    \text{FA}(\tau_\text{OOD}, &\tau_\text{ID}) = 
    \underbrace{\sum_{\substack{i \in \mathcal{D}_{\text{OOD}}}} 
    \mathbb{I}\big(i \in \mathcal{A}(\tau_\text{OOD}, \tau_\text{ID})\big)}_{\text{FA from accepted OOD}}
    \\&+
    \underbrace{\sum_{\substack{i \in \mathcal{D}_{\text{ID}}}} 
    \mathbb{I}\big(i \in \mathcal{A}(\tau_\text{OOD}, \tau_\text{ID})\big) 
    \cdot \mathbb{I}\big(m(x_i) \neq y_i\big)}_{\text{FA from accepted misclassified ID}}
\end{aligned}
    \end{equation*}
    
\keypoint{False Reject (FR):}
    ID samples that are not correctly accepted, either because they are rejected or misclassified when accepted.
    
    \begin{equation*}
\begin{aligned}
\text{FR}(\tau_\text{OOD}, \tau_\text{ID})
&= \sum_{i \in \mathcal{D}_{\text{ID}}}
\Biggl[
\underbrace{\mathbb{I}\big(i \notin \mathcal{A}(\tau_\text{OOD}, \tau_\text{ID})\big)}_{\text{Rejected ID}} \\
&+
\underbrace{\mathbb{I}\big(i \in \mathcal{A}(\tau_\text{OOD}, \tau_\text{ID})\big)
\cdot \mathbb{I}\big(m(x_i) \neq y_i\big)}_{\text{Misclassified ID}}
\Biggr]
\end{aligned}
\end{equation*}


Note that misclassified but accepted ID samples are counted in both FA and FR: they are ID but not correctly predicted (hence FR), and are also wrongly accepted (hence FA). This overlap is inherent in the double-scoring system, but does not lead to inconsistency.
Indeed, by construction, we have
\[
\text{TA} + \text{FR} = |\mathcal{D}_{\text{ID}}|, 
\qquad 
\text{TA} + \text{FA} = |\mathcal{A}(\tau_\text{OOD},\tau_\text{ID})|,
\]
so the resulting definitions of precision and recall are equivalent to their standard forms and can be expressed more compactly as
\begin{equation}
\begin{aligned}
\text{Precision}(\tau_\text{OOD},\tau_\text{ID}) &= \frac{\text{TA}(\tau_\text{OOD},\tau_\text{ID})}{|\mathcal{A}(\tau_\text{OOD},\tau_\text{ID})|}, 
\\
\text{Recall}(\tau_\text{OOD},\tau_\text{ID}) &= \frac{\text{TA}(\tau_\text{OOD},\tau_\text{ID})}{|\mathcal{D}_{\text{ID}}|}.
\end{aligned}
\end{equation}
The \textbf{DS-F1} is then defined as the maximum achievable F1-score across all possible threshold pairs, capturing the best possible system performance:
\begin{equation}
    \text{DS-F1} = \max_{\tau_\text{OOD}, \tau_\text{ID}} \left( \frac{2 \cdot \text{Precision}(\tau_\text{OOD}, \tau_\text{ID}) \cdot \text{Recall}(\tau_\text{OOD}, \tau_\text{ID})}{\text{Precision}(\tau_\text{OOD}, \tau_\text{ID}) + \text{Recall}(\tau_\text{OOD}, \tau_\text{ID})} \right)
    \label{eq:dual-f1}
\end{equation}
This formulation directly captures the best achievable joint operating point when both scoring functions are applied simultaneously, and serves as a point-wise evaluation metric complementary to area-based ones.  
As shown in Figure~\ref{fig:surface}, unlike the classical way to achieve the F1 score, which searches over a single threshold each time according to a single scoring method, DS-F1 searches over threshold pairs, thereby generalizing the decision boundary to a two-dimensional space.

\subsection{DS-AURC}
While DS-F1 identifies the best operating point, we also need to evaluate the system's performance across the entire range of possible operating points. 
To achieve this, we generalize the AURC metric and extend the notion of multi-threshold evaluation to the double-scoring setting. 

\input{figure2}

\keypoint{Coverage.}
Given an acceptance rule induced by a threshold $\tau$ of a scoring function $s$, according to Eq.~\ref{eq:setA}, the coverage $u$ is defined as the proportion of accepted samples ID samples by the system:
\begin{equation}
u = \frac{|\mathcal{A}(\tau)\cap \mathcal{D}_{\text{ID}}|}{|\mathcal{D}_{\text{ID}}|},
\label{eq:coverage}
\end{equation}
This definition is consistent with the ID-only case, since when $\mathcal{D}_{\text{OOD}}=\varnothing$, Eq.~\ref{eq:coverage} reduces to the standard coverage definition in selective classification.

\keypoint{AURC on ID-only test set.}
Let $Z = \mathbb{I}(m(x)\neq y)$ denote the indicator for misclassification. We denote $\text{SelectiveRisk}$ by $\text{SR}$ for brevity. For an evaluation set containing only ID data, the selective risk at threshold $\tau$ and the resulting AURC are defined as
\begin{equation}
\label{eq:aurc-id}
\begin{aligned}
\text{SR}(\tau) &= \frac{\sum_{i\in\mathcal{D}_{\text{ID}}} Z_i\cdot \mathbb{I}(i\in\mathcal{A}(\tau))}{|\mathcal{A}(\tau)|}, \\
\text{AURC} &= \int_0^1 \text{SelectiveRisk}(\tau(u))\, du,
\end{aligned}
\end{equation}
where the ratio is defined as $0$ when $|\mathcal{A}(\tau)|=0$, and $\tau(u)$ is the unique threshold corresponding to coverage $u$ under a single scoring function. 

\keypoint{AURC on ID+OOD evaluation set.}
When we consider a more realistic case, the evaluation set contains both ID and OOD data, and $\mathcal{A}(\tau)$ may include OOD samples. 
In this case, the selective risk must be redefined to reflect the accuracy of the ID portion of accepted samples:
\begin{equation}
\text{SR}(\tau) 
= \frac{\sum_{i\in\mathcal{D}_{\text{ID}}} Z_i\cdot \mathbb{I}(i\in\mathcal{A}(\tau)) \;+\; |\mathcal{A}(\tau)\cap\mathcal{D}_{\text{OOD}}|}
       {|\mathcal{A}(\tau)|}
\label{eq:selrisk-id-ood}
\end{equation}
This definition treats not only accepted misclassified ID as risks but also each accepted OOD sample, since once the system accepts a prediction given by an OOD input, it indicates a failure. Note that when $\mathcal{D}_{\text{OOD}}=\varnothing$, Eq.~\ref{eq:selrisk-id-ood} reduces to the selective risk in Eq.~\ref{eq:aurc-id}, so this definition is compatible with the ID-only testing case. Meanwhile, the calculation of AURC will be the same as in Eq.~\ref{eq:aurc-id}.

\keypoint{DS-AURC.}
We still consider the ID+OOD evaluation setting here. When two scoring functions $s_\text{OOD}$ and $s_\text{ID}$ are jointly applied, a given coverage level $u$ may correspond to multiple threshold pairs $(\tau_\text{OOD},\tau_\text{ID})$, 
each inducing a distinct acceptance set $\mathcal{A}(\tau_\text{OOD},\tau_\text{ID})$. 
We define the selective risk for a threshold pair similar to Eq.~\ref{eq:selrisk-id-ood} as follows:
\begin{equation}
\begin{aligned}
&\text{SR}(\tau_\text{OOD},\tau_\text{ID})=\\ 
&\frac{\sum_{i\in\mathcal{D}_{\text{ID}}} Z_i\cdot \mathbb{I}(i\in\mathcal{A}(\tau_{\text{OOD}},\tau_{\text{ID}}))
       \;+\; |\mathcal{A}(\tau_{\text{OOD}},\tau_{\text{ID}})\cap\mathcal{D}_{\text{OOD}}|}
       {|\mathcal{A}(\tau_{\text{OOD}},\tau_{\text{ID}})|}
\label{eq:selrisk-dual}
\end{aligned}
\end{equation}
Since at a fixed coverage $u$ there may exist multiple threshold pairs, the selective risk is no longer unique but a set of values. 
With Eq.~\ref{eq:selrisk-dual}, we therefore define the DS-AURC as
\begin{equation}
\label{eq:dual-aurc-prob}
\begin{aligned}
\text{DS-AURC} \text{=} \int_0^1 \phi \Big\{\text{SR}(\tau_\text{OOD},\tau_\text{ID}) \text{:}(\tau_\text{OOD},\tau_\text{ID}) \in \mathbb{R}^2, \\
\frac{|\mathcal{A}(\tau_\text{OOD},\tau_\text{ID})\cap \mathcal{D}_{\text{ID}}|}{|\mathcal{D}_{\text{ID}}|} \text{=} u \Big\} \, du
\end{aligned}
\end{equation}
\noindent where $\phi$ is an aggregation operator over risks associated with the same coverage level, and the definition of the coverage here is also consistent with Eq.~\ref{eq:coverage}. In practice, AUC is often computed via the trapezoidal rule, which by default takes the last item when multiple risks share the same coverage. We instead set $\phi=\min$ to capture the best achievable risk at each coverage level. Consequently, the DS-AURC reflects the smallest possible coverage–risk area, highlighting the optimistic bound of model performance under selective prediction.

\keypoint{Summary and properties of DS-F1 and DS-AURC.}
The proposed DS-F1 and DS-AURC metrics provide a principled generalization of the classical single-threshold evaluation pipeline. 
Their detailed implementation, including pseudo-code for computing DS-F1 and DS-AURC, is provided in Appendix~\ref{section:supp_implementation}.
Meanwhile, as formally shown in Appendix~\ref{section:supp_proof}, DS-F1 is guaranteed to be at least as high as the standard F1, and DS-AURC is guaranteed to be lower than or equal to the standard AURC. This ensures that adopting double scoring never worsens evaluation outcomes. 

These metrics exhibit several notable properties. First, they naturally generalize the traditional metrics: when one threshold is fixed, DS-F1 and DS-AURC reduce to the standard F1 and AURC, maintaining full consistency with previous pipelines. Second, our metrics provide clear performance guarantees: DS-F1 is always greater than or equal to F1, and DS-AURC is always less than or equal to AURC. Third, double scoring can change the ranking of methods. For example, a model with the lowest single-score AURC may not achieve the lowest DS-AURC, since DS-AURC evaluates multiple threshold combinations at each coverage level. As shown in Figure~\ref{fig:aurc_min}, for each coverage bin, DS-AURC selects the minimal risk (\textcolor{orange}{orange} points), resulting in a risk surface that is always equal to or better than that of single scoring.

Taken together, DS-F1 and DS-AURC capture complementary but fundamentally different aspects of reliability. 
DS-F1 reflects the best achievable operating point under an optimal threshold configuration, whereas DS-AURC evaluates the model’s consistency across different levels of selectivity by aggregating performance over multiple threshold pairs. 
As a result, a method may achieve strong peak performance while failing to maintain low risk when operating conditions change, a behavior that is explicitly exposed by DS-AURC but obscured by single-threshold evaluation. 
This distinction highlights the necessity of jointly considering both point-wise optimality and global robustness when evaluating systems that must handle OOD detection and failure prediction in tandem.

\input{tables/table1_posthoc}

%% file: figure2.tex
\begin{figure*}[t]
\centering
\includegraphics[width=0.8\linewidth]{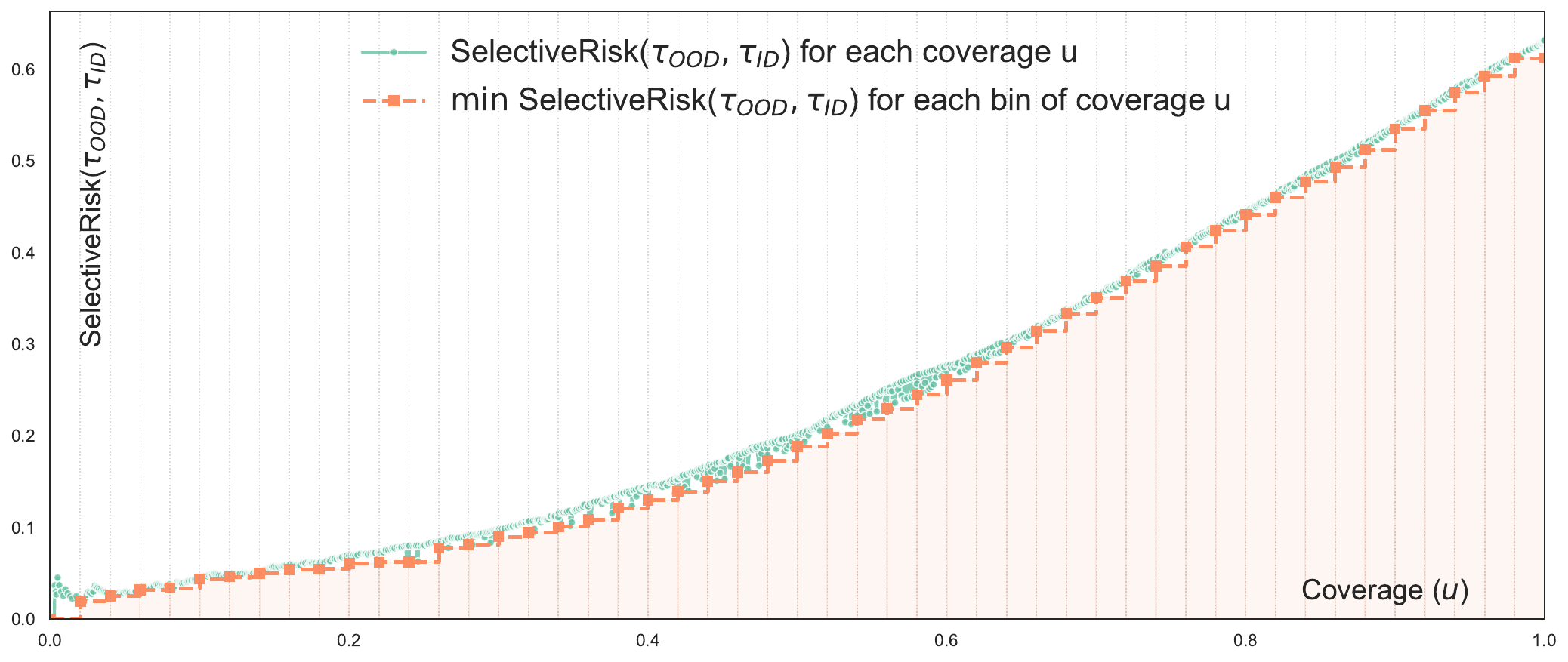}
\caption{\textbf{Example of DS-AURC.} Selective risk versus coverage for double scoring is shown in \textcolor{green}{green}. DS-AURC uses the best risk in each coverage bin (\textcolor{orange}{orange}, in Eq.~\ref{eq:dual-aurc-prob}), resulting in a lower (better) AURC. Results with ResNet-18 on CIFAR-100 (ID) and CIFAR-10 (OOD), using VIM~\cite{wang2022vim} for the OOD score and MSP~\cite{hendrycks2016baseline} for the ID score.}
\label{fig:aurc_min}
\end{figure*}

%% file: tables/table1_posthoc.tex
\begin{table*}[t]
\centering
\scalebox{0.9}{
\begin{tabular}{l||c|c|c|c||l|l|l|l} 
\toprule
\multirow{3}{*}{\textbf{Method}} & \multicolumn{4}{l||}{\begin{tabular}[c]{@{}l@{}}\textbf{Model}: ResNet18~\\\textbf{Training/Validation set}: CIFAR100 training set\\\textbf{Evaluation set}: CIFAR100 test set + Near/Far OOD sets\end{tabular}} & \multicolumn{4}{l}{\begin{tabular}[c]{@{}l@{}}\textbf{Model}: DeiT-B\\ \textbf{Training/Validation set}: ImageNet1K training set\\\textbf{Evaluation set}: ImageNet1K test set + Near/Far OOD sets\end{tabular}} \\ 
\cmidrule{2-9}
 & \multicolumn{2}{c|}{\textbf{Single Scoring}} & \multicolumn{2}{c||}{\textbf{Double Scoring}} & \multicolumn{2}{c|}{\textbf{Single Scoring}} & \multicolumn{2}{c}{\textbf{Double Scoring}} \\ 
\cmidrule{2-9}
 & \textbf{F1} $\uparrow$ & \textbf{AURC} $\downarrow$ & \textbf{DS-F1} $\uparrow$ & \textbf{DS-AURC} $\downarrow$ & \multicolumn{1}{c|}{\textbf{F1} $\uparrow$ } & \multicolumn{1}{c|}{\textbf{AURC} $\downarrow$} & \multicolumn{1}{c|}{\textbf{DS-F1} $\uparrow$} & \multicolumn{1}{c}{\textbf{DS-AURC} $\downarrow$} \\ 
\midrule
MSP~\cite{hendrycks2016baseline} & \textbf{\textcolor{blue5}{67.42}} / 57.03 & 202.59 / 368.07 & 67.42 / 57.03 & 202.38 / 367.56 & \textbf{\textcolor{blue5}{70.55}} / \textbf{\textcolor{blue4}{79.05}} & \textbf{\textcolor{blue5}{241.44}} / \textbf{\textcolor{blue4}{98.15}} & 70.55 / 79.05 & 241.35 / 98.13 \\
OpenMax~\cite{bendale2016towards} & 64.96 / 55.41 & 264.21 / 444.53 & 67.42 / 57.09 & 201.78 / 365.87 & 67.06 / 77.73 & 404.98 / 248.49 & \textbf{\textcolor{blue5}{70.70}} / \textbf{\textcolor{blue5}{79.60}} & 239.42 / \textbf{\textcolor{blue1}{93.32}} \\
ODIN~\cite{liang2017enhancing} &  65.89 / 56.54 & 219.63 / 358.68 & 67.42 / \textbf{\textcolor{blue3}{58.75}} & 199.82 / \textbf{\textcolor{blue3}{330.49}} & 67.48 / 76.21 & 322.35 / 192.30 & 70.55 / 79.05 & 239.54 / 97.16 \\
MDS~\cite{lee2018simple} & 56.19 / 47.82 & 339.54 / 464.67 & 67.43 / 57.22 & 201.78 / 355.07 & 68.60 / 78.12 & 252.45 / \textbf{\textcolor{blue3}{99.23}} & \textbf{\textcolor{blue3}{70.56}} / \textbf{\textcolor{blue3}{79.26}} & \textbf{\textcolor{blue5}{226.20}} / \textbf{\textcolor{blue5}{77.30}} \\
Gram~\cite{hendrycks2021unsolved} & 53.23 / 46.09 & 536.12 / 548.61 & 67.42 / \textbf{\textcolor{blue5}{59.16}} & 201.14 / \textbf{\textcolor{blue5}{313.08}} & 64.87 / 73.50 & 488.21 / 424.83 & 70.55 / 79.05 & 240.89 / 97.43 \\
EBO~\cite{liu2020energy} & 66.62 / \textbf{\textcolor{blue2}{57.41}} & 205.79 / \textbf{\textcolor{blue4}{345.83}} & 67.42 / 57.76 & 197.99 / 342.04 & 66.92 / 76.20 & 355.94 / 214.35 & 70.55 / 79.05 & 240.78 / 97.78 \\
GradNorm~\cite{huang2021importance} & 63.31 / 47.95 & 322.48 / 518.60 & \textbf{\textcolor{blue3}{67.45}} / 57.05 & 201.08 / 365.74 & 64.87 / 73.51 & 532.10 / 454.47 & 70.55 / 79.05 & 240.78 / 97.78 \\
ReAct~\cite{sun2021react} & 66.32 / \textbf{\textcolor{blue5}{58.03}} & 213.53 / \textbf{\textcolor{blue5}{335.99}} & 67.43 / \textbf{\textcolor{blue2}{58.44}} & 199.61 / \textbf{\textcolor{blue2}{331.40}} & 67.99 / 77.67 & 290.84 / 131.86 & \textbf{\textcolor{blue2}{70.56}} / 79.05 & \textbf{\textcolor{blue1}{239.29}} / 93.36 \\
MLS~\cite{basart2022scaling} & 66.90 / \textbf{\textcolor{blue4}{57.47}} & 203.87 / \textbf{\textcolor{blue3}{346.30}} & 67.42 / 57.71 & 197.76 / 343.24 & 69.28 / 78.36 & 303.42 / 157.43 & 70.55 / 79.05 & 240.66 / 98.10 \\
KLM~\cite{basart2022scaling} & 65.52 / 48.90 & 306.23 / 524.84 & \textbf{\textcolor{blue1}{67.44}} / 57.09 & 199.23 / 351.05 &  \textbf{\textcolor{blue4}{70.54}} / \textbf{\textcolor{blue5}{79.37}} & 249.44 / \textbf{\textcolor{blue5}{93.00}} & \textbf{\textcolor{blue4}{70.63}} / \textbf{\textcolor{blue4}{79.41}} & \textbf{\textcolor{blue3}{229.50}} / \textbf{\textcolor{blue2}{83.42}} \\
VIM~\cite{wang2022vim} & 59.91 / 55.07 & 270.50 / 359.40 & \textbf{\textcolor{blue4}{67.46}} / \textbf{\textcolor{blue4}{58.94}} & 198.39 / \textbf{\textcolor{blue4}{321.09}} & 68.21 / 77.92 & 299.90 / 139.83 & \textbf{\textcolor{blue1}{70.55}} / \textbf{\textcolor{blue2}{79.18}} & \textbf{\textcolor{blue2}{231.73}} / \textbf{\textcolor{blue3}{81.86}} \\
KNN~\cite{sun2022out} & 67.10 / \textbf{\textcolor{blue4}{57.47}} & 220.78 / \textbf{\textcolor{blue2}{347.05}} & \textbf{\textcolor{blue5}{67.67}} / \textbf{\textcolor{blue1}{58.26}} & \textbf{\textcolor{blue1}{197.30}} / \textbf{\textcolor{blue1}{334.68}} & 67.91 / 77.81 & 261.15 / 102.68 & 70.55 / \textbf{\textcolor{blue1}{79.10}} & \textbf{\textcolor{blue4}{229.33}} / \textbf{\textcolor{blue4}{79.05}} \\
DICE~\cite{sun2022dice} & 65.82 / 57.11 & 214.11 / \textbf{\textcolor{blue1}{347.94}} & 67.43 / 58.04 & 198.75 / 338.83 & 64.96 / 74.14 & 464.09 / 313.53 & 70.55 / 79.06 & 239.46 / 96.42 \\
SIRC (MSP, $||z||_1$)~\cite{xia2022augmenting} & \textbf{\textcolor{blue2}{67.36}} / 56.96 & \textbf{\textcolor{blue2}{199.20}} / 354.96 & 67.44 / 57.05 & \textbf{\textcolor{blue3}{197.08}} / 353.41 & \textbf{\textcolor{blue2}{70.43}} / \textbf{\textcolor{blue2}{78.87}} & \textbf{\textcolor{blue3}{244.86}} / \textbf{\textcolor{blue1}{102.46}} & 70.55 / 79.05 & 240.78 / 97.78 \\
SIRC (MSP, Res.)~\cite{xia2022augmenting} & \textbf{\textcolor{blue1}{67.33}} / 57.00 & \textbf{\textcolor{blue3}{198.86}} / 354.72 & 67.42 / 57.03 & \textbf{\textcolor{blue2}{197.21}} / 353.08 & \textbf{\textcolor{blue1}{70.42}} / \textbf{\textcolor{blue1}{78.86}} & \textbf{\textcolor{blue2}{245.16}} / 102.87 & 70.55 / 79.05 & 241.04 / 98.09 \\
SIRC (-H, $||z||_1$)~\cite{xia2022augmenting} & 67.24 / 57.18 & \textbf{\textcolor{blue5}{198.46}} / 349.53 & \textbf{\textcolor{blue2}{67.44}} / 57.25 & \textbf{\textcolor{blue5}{196.46}} / 348.47 & 69.74 / 78.80 & 264.08 / 118.81 & 70.55 / 79.05 & 241.06 / 98.08 \\
SIRC (-H, Res.)~\cite{xia2022augmenting} & 67.20 / \textbf{\textcolor{blue1}{57.20}} & \textbf{\textcolor{blue4}{198.66}} / 349.68 & 67.42 / 57.26 & \textbf{\textcolor{blue4}{196.60}} / 348.36 & 69.74 / 78.80 & 267.67 / 118.81 & 70.55 / 79.05 & 241.04 / 98.05 \\
Plug-in BB ($||z||_1$)~\cite{narasimhan2024plugin}  & \textbf{\textcolor{blue5}{67.42}} / 57.03 & 202.61 / 368.10 & 67.43 / 57.03 & 202.38 / 367.57 & \textbf{\textcolor{blue3}{70.48}} / \textbf{\textcolor{blue3}{79.00}} & \textbf{\textcolor{blue4}{244.49}} / \textbf{\textcolor{blue2}{101.96}} & 70.55 / 79.05 & 241.04 / 98.01 \\
Plug-in BB (Res)~\cite{narasimhan2024plugin} & \textbf{\textcolor{blue3}{67.40}} / 57.01 & \textbf{\textcolor{blue1}{199.30}} / 357.43 & 67.42 / 57.03 & 198.80 / 356.88 & 69.95 / 78.84 &\textbf{\textcolor{blue1}{245.49}} / 102.55 & 70.55 / 79.05 & 241.04 / 98.10 \\
\midrule
\textbf{\textbf{ID Acc.}} & \multicolumn{4}{c||}{77.32} & \multicolumn{4}{c}{81.79} \\
\bottomrule
\end{tabular}
}
\caption{\textbf{Experiments are conducted on CIFAR100~\cite{krizhevsky2009learning} with ResNet-18~\cite{he2016deep}, and on ImageNet~\cite{deng2009imagenet} with DeiT-B~\cite{touvron2021training}.} For double scoring metrics, we use MSP as the ID score ($s_{\rm ID}$) and apply different OOD scores ($s_{\rm OOD}$), reporting results on both Near- and Far-OOD tests. CIFAR-100 experiments are repeated three times, and average results are reported. The top five methods for each metric are highlighted with a color gradient from \textbf{\textcolor{blue1}{light blue}} to \textbf{\textcolor{blue5}{dark blue}}.}
\label{tab1:two_score}
\end{table*}

%% file: sections/4_method_2.tex
\section{SURE+: A Unified Training Framework for Reliable Classification}

This section presents SURE+, a unified training framework for building strong and reliable classification models under both in-distribution (ID) and out-of-distribution (OOD) evaluation, a setting that commonly arises in real-world applications. The key insight behind SURE+ is that ID confidence calibration and OOD robustness are best addressed through a coherent and unified training framework. Building upon SURE~\cite{li2024sure}, which has proven effective for failure prediction, we systematically explore each component from the perspective of joint reliability and retain only those that demonstrably contribute to performance on both ID and OOD data. This leads to a training recipe that is simpler, more stable, and more generalizable.

\keypoint{Background: Revisiting SURE under joint ID+OOD evaluation.} SURE~\cite{li2024sure} integrates multiple techniques: RegMixup~\cite{pinto2022using}, a cosine similarity classifier~\cite{gidaris2018dynamic}, SAM and SWA~\cite{foret2020sharpness,izmailov2018averaging}, and Correctness Ranking Loss (CRL)~\cite{moon2020confidence}, to improve in-distribution (ID) accuracy and failure prediction. However, our empirical analysis reveals a key limitation: Its performance does not consistently extend to out-of-distribution (OOD) detection, particularly under joint ID+OOD evaluation protocols, we show the experimental results for SURE on both ID and OOD data in Table~\ref{tab:table2_r18_ds} and Table~\ref{tab:table3_dinov3_ds}. This motivates a systematic rethinking of the training pipeline, focusing on identifying components that reliably improve performance across both ID and OOD settings.

\keypoint{Data augmentation as a unifying regularization.} Robust data augmentation is a core component of SURE+. Motivated by the observation that invariance to data perturbations benefits both failure prediction and OOD detection, we adopt a complementary augmentation strategy. Specifically, (i) RegMixup~\cite{pinto2022using} regularizes decision boundaries through label-preserving feature interpolation, and (ii) RegPixMix, inspired by PixMix~\cite{hendrycks2022pixmix}, enforces robustness under stronger pixel-level perturbations. By jointly applying RegMixup and RegPixMix, SURE+ promotes consistency under both semantic-level and pixel-level variations, providing a unified regularization mechanism that improves overall robustness.

\keypoint{Optimization for Reliability: Sharpness-Aware and Stable Training.} From an optimization perspective, SURE+ focuses on flat minima and training stability, both of which are important for reliable uncertainty estimation. We adopt F-SAM~\cite{li2024friendly}, a refined variant of SAM~\cite{foret2020sharpness}, which retains the benefits of sharpness-aware optimization while reducing training instability and over-regularization. Compared to standard SAM, F-SAM leads to more consistent convergence and better robustness to label noise and input perturbations, which are key properties for joint ID and OOD reliability.

\keypoint{Model ensembling via EMA and re-normalized batch norm~\cite{Moralesema}.} To stabilize model predictions, SURE+ replaces stochastic weight averaging (SWA)~\cite{izmailov2018averaging} with an exponential moving average (EMA) of model parameters. EMA preserves the ensemble-like smoothing effect of SWA while (i) reducing sensitivity to hyperparameter schedules, (ii) simplifying training and implementation, and (iii) providing more stable behavior under mixed ID and OOD distributions. When combined with re-normalized batch normalization (Re-BN) statistics, reliability is further and consistently improved across evaluation settings.

\keypoint{Summary.} Complementary data augmentation, sharpness-aware optimization, and Re-BN EMA ensembling together form SURE+, a streamlined and unified training framework for reliable classification. Rather than emphasizing novelty in individual modules, SURE+ shows that carefully selected and well-integrated training principles can provide a strong and reproducible baseline for joint ID and OOD reliability evaluation. This unified perspective makes SURE+ a practical foundation for future research on reliable classification systems. We also remove the Correctness Ranking Loss (CRL)~\cite{moon2020confidence} and Cosine Similarity Classifier (CSC) used in SURE, as they offer marginal improvements while introducing additional complexity.

%% file: sections/5_exp.tex
\section{Experiments}

\subsection{Datasets and Architectures} 

\keypoint{Datasets.}
We conduct experiments on CIFAR-100~\cite{krizhevsky2009learning} and ImageNet-1K~\cite{deng2009imagenet}, following the OpenOOD benchmark~\cite{yang2022openood}.
For CIFAR-100, the Near-OOD datasets include CIFAR-10~\cite{krizhevsky2009learning} and TinyImageNet~\cite{le2015tiny}, with 2,502 overlapping images removed,
while the Far-OOD datasets include MNIST~\cite{lecun1998mnist}, SVHN~\cite{netzer2011reading}, Textures~\cite{cimpoi2014describing}, and Places365~\cite{zhou2017places}, with 1,305 overlapping images removed.
For ImageNet-1K, we include SSB-hard~\cite{vaze2022openset} and NINCO~\cite{bitterwolf2023or} as Near-OOD datasets.
SSB-hard contains 49,000 images covering 980 categories selected from ImageNet-21K~\cite{ridnik1imagenet},
while NINCO consists of 5,879 manually curated images.
The Far-OOD datasets include iNaturalist~\cite{van2018inaturalist}, Textures~\cite{cimpoi2014describing}, and OpenImage-O~\cite{wang2022vim}.

\keypoint{Archtectures.} Following the OpenOOD benchmark, we conduct experiments on CIFAR-100 as the in-distribution (ID) dataset using ResNet-18~\cite{he2016deep} for small-scale evaluation. For large-scale experiments with ImageNet as the ID dataset, we adopt the transformer-based architecture DINOv3 ViT-L/16~\cite{simeoni2025dinov3}. Experiments on CIFAR-100 are repeated three times, and we report the average results. 

\keypoint{Training settings.} We unify the optimization settings across all methods, including a batch size of 128, SGD with momentum 0.9, an initial learning rate of 0.05 with cosine annealing~\cite{loshchilov2017sgdr}, and a weight decay of 5e-3. Models are trained for 200 epochs, and results on CIFAR-100 are averaged over three independent runs to reduce stochastic variability.
For ImageNet-1K, models are fine-tuned for 20 epochs with a maximum learning rate of 1e-6.

\keypoint{{Evaluation settings.}} 
We evaluate the models and post-hoc methods on the corresponding ID test sets, along with Near/Far-OOD sets, reporting both standard metrics (AURC, F1) and the proposed metrics (DS-AURC, DS-F1). 
Different from OpenOOD~\cite{yang2022openood,zhang2023openood}, which only evaluates ID--OOD discrimination, SURE~\cite{li2024sure}/FMFP~\cite{zhu2023revisiting}, which only considers ID test set failure prediction, SIRC~\cite{xia2022augmenting}, which treats the two problems simultaneously but separately, while Plug-in BB~\cite{narasimhan2024plugin} combines the two objectives under a unified framework but is limited to a single confidence score. In contrast, our setting directly evaluates selective classification when both ID and OOD samples coexist.

For the proposed double-scoring metrics, we adopt MSP~\cite{hendrycks2016baseline} as one scoring function, while the other is provided by different post-hoc methods, such as ODIN~\cite{liang2017enhancing}, EBO~\cite{liu2020energy}, etc. 
When the post-hoc method itself is MSP, our metrics naturally reduce to the single-scoring case. 
For the single-scoring metrics, we simply use the score from each post-hoc method as the sole scoring measure.
All hyperparameters of post-hoc methods are selected automatically using both ID and OOD validation samples, with the configuration maximizing AUROC used for final testing. Throughout this paper, unless otherwise noted, all metrics are multiplied by 100, while AURC is multiplied by 1000. 
For completeness, results under the OpenOOD protocol are provided in Appendix~\ref{section:supp_openood}, highlighting the OOD discriminative ability of post-hoc methods and confirming the consistency of our implementation with standard OpenOOD benchmarks.

\input{tables/table8_threshold}
\input{tables/table2_r18_ds}
\input{tables/table3_dinov3_ds}

\subsection{Comparison Between Double Scoring and Single Scoring}
We first compare the proposed double-scoring framework with conventional single-scoring approaches. The results are summarized in Table~\ref{tab1:two_score}. For single-scoring methods, we report AURC and F1 on the test set that contains both ID and OOD samples. For double scoring, we evaluate the proposed DS-AURC and DS-F1 on the same test set. From the results, we make three key observations. 

First, \emph{double scoring consistently outperforms single scoring} across all datasets and evaluation settings.
By combining two complementary scores, namely the confidence score from a post-hoc OOD detection method and the MSP score~\cite{hendrycks2016baseline}, the proposed double-scoring strategy achieves clear and consistent improvements in both F1 and AURC.
These results validate the effectiveness of DS-F1 and DS-AURC as evaluation metrics for selective classification when ID and OOD samples coexist. 

Second, \emph{the performance gains are particularly pronounced on Far-OOD datasets}. Compared to the MSP baseline, double scoring yields substantially larger improvements on Far-OOD samples than on Near-OOD samples. This suggests that existing OOD detection methods are more effective at identifying visually distinct OOD data, while their impact is reduced in the more challenging Near-OOD setting. This observation is consistent with findings reported in OpenOOD~\cite{zhang2023openood}. Although Near-OOD detection is generally considered more difficult, most existing methods are primarily designed and evaluated on Far-OOD benchmarks, which has led to comparatively limited progress on Near-OOD scenarios.

Third, \emph{MSP remains a strong single-score baseline}. Even without double scoring, MSP demonstrates competitive performance, especially on large-scale datasets such as ImageNet~\cite{deng2009imagenet}. This highlights that while MSP provides a robust baseline, the proposed double-scoring framework is able to further push performance beyond the limits of single-score methods.

On the CIFAR-100 benchmark, we observe that Near-OOD performance surpasses Far-OOD, which is consistent with the findings in OpenOOD \cite{zhang2023openood}. Although Near-OOD detection is generally regarded as more challenging, most existing methods are designed and evaluated on Far-OOD, leading to comparatively limited progress on Near-OOD.

\subsection{Double-Threshold Selection and Generalization Analysis}
To further verify that double scoring is consistently better than single scoring, we design a controlled experiment with explicit validation and test splits. We train ResNet-18 on the CIFAR-100~\cite{krizhevsky2009learning} training set, using CIFAR-100 validation as ID with CIFAR-10~\cite{krizhevsky2009learning} as OOD for threshold selection, and test on the CIFAR-100 test set as ID with MNIST~\cite{lecun1998mnist} as OOD. We adopt ReAct~\cite{sun2021react} as the OOD score and compare single score thresholds with a joint DS-F1 threshold pair, where MSP~\cite{hendrycks2016baseline} is the ID score.

As shown in Table~\ref{tab:double_scoring_val}, we first train a model on CIFAR-100 with standard cross entropy, then search for the best ID score threshold, OOD score threshold, and DS-F1 threshold pair on the validation split to maximize F1. The joint double scoring reaches 33.2, which is higher than either single scoring method, as expected. We then apply the validation thresholds to the test split to simulate real use and observe similar relative performance. Double scoring remains better than single scoring, showing that DS-F1 is consistent and better than single F1 across validation and test. This reduces the gap between threshold selection and real deployment, which is important in practice.

\subsection{Effectiveness of SURE+}
The effectiveness of the proposed SURE+ framework is first evaluated on CIFAR100 using ResNet-18. Table~\ref{tab:table2_r18_ds} reports both DS-F1 and DS-AURC results under Near-OOD and Far-OOD settings, covering multiple training methods and post-hoc scores. The compared training strategies include Baseline (standard cross-entropy loss), Mixup~\cite{thulasidasan2019mixup}, Pixmix~\cite{hendrycks2022pixmix}, Augmix~\cite{hendrycks2019augmix}, Cutmix~\cite{yun2019cutmix}, RegMixup~\cite{pinto2022using}, and SURE~\cite{li2024sure}, following the OpenOOD benchmark~\cite{yang2022openood,zhang2023openood}. 

As shown in Table~\ref{tab:table2_r18_ds}, SURE+ consistently achieves the best performance across both DS-F1 and DS-AURC metrics for Near-OOD and Far-OOD detection. In addition, SURE+ reaches an in-distribution classification accuracy of 81.66\% on CIFAR100, outperforming all competing training paradigms. Notably, these gains hold across the majority of post-hoc scores, indicating that the advantages of SURE+ are robust and largely independent of the specific post-hoc choice.

To further examine scalability and generalization, SURE+ is evaluated on a larger-scale and more challenging setting by fine-tuning DINOv3 ViT-L/16 on ImageNet-1K with an input resolution of $224 \times 224$. The corresponding results are reported in Table~\ref{tab:table3_dinov3_ds}, where DS-F1 and DS-AURC are jointly presented for both Near-OOD and Far-OOD benchmarks. In this setting, SURE+ achieves a test accuracy of 88.49\% on ImageNet-1K, while again delivering the best performance on both DS metrics across most post-hoc methods. These results demonstrate that the effectiveness of SURE+ is not limited to CNN architectures, and extends naturally to transformer-based models in large-scale settings.

Finally, an ablation study is conducted to analyze the contribution of each component in the SURE+ framework. The results show that every component contributes meaningfully to improving reliability, and removing any single component leads to a noticeable degradation in DS performance. When combined, the full SURE+ pipeline yields substantial improvements over SURE and other strong baselines, confirming the complementary nature of its design. Results for single-task settings, including OOD detection and failure prediction, are provided in Appendix~\ref{app_sec:surep} for completeness.

%% file: tables/table8_threshold.tex
\begin{table}[t]
\centering

\begin{tabular}{lccc}
\toprule
\multicolumn{4}{c}{\textbf{Validation (CIFAR-100 validation + CIFAR-10)}} \\
\midrule
Method & MSP $\tau_{\text{ID}}$ & ReAct $\tau_{\text{OOD}}$ & F1 / DS-F1  \\
\midrule
ID-only &  0.952 & - & 31.5 \\
OOD-only & - & 0.399 & 31.7 \\
Double Scoring  & 0.927 & 0.345 & \textbf{33.2} \\
\midrule
\multicolumn{4}{c}{\textbf{Test (CIFAR-100 test set + MNIST)}} \\
\midrule
Method & MSP $\tau_{\text{ID}}$ & ReAct $\tau_{\text{OOD}}$ & F1 / DS-F1 \\
\midrule
ID-only (val) & 0.952 & - & 24.4 \\
OOD-only (val) & - & 0.399 & 24.9 \\
Double Scoring (val) & 0.927 & 0.345 & \textbf{26.0} \\
\textcolor{gray}{ID-only (test-opt)} & \textcolor{gray}{0.766} & \textcolor{gray}{-} & \textcolor{gray}{27.2} \\
\textcolor{gray}{OOD-only (test-opt)} & \textcolor{gray}{-} & \textcolor{gray}{0.297} & \textcolor{gray}{30.3} \\
\textcolor{gray}{Double Scoring (test-opt)} & 
\textcolor{gray}{0.371} & 
\textcolor{gray}{0.261} & 
\textcolor{gray}{\textbf{31.9}} \\
\bottomrule
\end{tabular}
\caption{F1/DS-F1 scores and thresholds for MSP~\cite{hendrycks2016baseline}, ReAct~\cite{sun2021react} under validation and test. \textcolor{gray}{Test-opt} indicates the optimal F1/DS-F1 score calculated according to the thresholds directly selected on the test set.}
\label{tab:double_scoring_val}
\end{table}

%% file: tables/table2_r18_ds.tex
\begin{table*}[!ht]
\centering
\scalebox{0.76}{
\begin{tabular}{l||cccccccccc} 
\toprule
\multicolumn{11}{l}{\textbf{Metric: DS-F1 $\uparrow$ ResNet-18 - Trained on CIFAR-100 training set. Evaluated on CIFAR100 test set + Near/Far-OOD.}} \\ 
\midrule
\textbf{Training strategy} & MSP & OpenMax & MDS & Gram & ReAct & KLM & VIM & KNN & SIRC & Acc. \\ 
\midrule
Basic & 67.42/57.03 & 67.42/57.09 & 67.43/57.22 & 67.42/59.16 & 67.43/58.44 & 67.44/57.09 & 67.46/58.94 & 67.67/58.26 & 67.44/57.05 & 77.32 \\
Mixup & 67.86/58.02 & 67.98/61.29 & 67.87/58.34 & 67.87/58.04 & 67.87/58.03 & 67.88/58.18 & 67.87/58.53 & 68.06/59.72 & 67.87/58.02 & 78.47 \\
RegMixup & 68.77/55.08 & 68.77/58.16 & 68.77/55.41 & 68.77/56.44 & 68.78/55.83 & 68.78/55.15 & 68.77/55.60 & 68.81/58.22 & 68.78/55.11 & 79.35 \\
AugMix & 66.40/58.82 & 66.40/58.93 & 66.40/58.92 & 66.40/59.22 & 66.41/60.19 & 66.41/58.91 & 66.42/59.39 & 66.66/59.26 & 66.41/58.86 & 76.98 \\
PixMix & 66.65/57.06 & 66.65/57.04 & 66.65/57.05 & 66.65/60.38 & 66.78/58.83 & 66.68/57.40 & 66.69/58.13 & 66.84/58.25 & 66.67/57.53 & 77.20 \\
CutMix & 65.99/55.59 & 66.08/60.75 & 65.99/56.30 & 65.99/55.99 & 65.99/55.58 & 66.00/55.88 & 65.99/56.06 & 66.59/59.83 & 65.99/55.58 & 77.81 \\
\midrule
SURE & 68.07/53.09 & 68.18/54.10 & 68.07/53.13 & 68.07/56.89 & 68.08/54.00 & 68.07/53.31 & 68.07/53.50 & 69.43/57.26 & 68.37/55.81 & 80.55 \\
- CSC & 68.72/56.17 & 68.72/58.58 & 68.72/56.94 & 68.73/57.04 & 69.05/60.45 & 68.85/58.27 & 68.74/60.16 & 68.76/60.49 & 68.72/56.20 & 80.36 \\
- CRL & 69.03/57.56 & 69.02/58.49 & 69.02/58.00 & 69.02/58.40 & 69.21/60.22 & 69.08/58.82 & 69.03/\textbf{62.49} & 69.07/62.83 & 69.03/57.65 & 80.68 \\
+ SWA$\rightarrow$EMA & 69.49/54.47 & 69.49/57.65 & 69.49/54.53 & 69.49/54.95 & 69.49/55.48 & 69.65/55.91 & 69.49/54.92 & 69.49/57.40 & 69.49/54.48 & 80.63 \\
+ ReBN & 69.24/57.34 & 69.24/58.24 & 69.24/57.77 & 69.24/57.99 & 69.45/60.71 & 69.29/58.40 & 69.26/60.01 & 69.27/62.48 & 69.24/57.36 & 80.54 \\
+ SAM$\rightarrow$F-SAM & 69.41/57.05 & 69.41/59.00 & 69.41/57.45 & 69.42/58.46 & 69.52/60.26 & 69.53/58.54 & 69.41/58.59 & 69.45/61.52 & 69.41/57.09 & 80.79 \\
+ RegPixMix (\textbf{SURE+}) & \textbf{70.67}/\textbf{61.35} & \textbf{70.67}/\textbf{61.33} & \textbf{70.67}/\textbf{61.42} & \textbf{70.67}/\textbf{64.76} & \textbf{70.76}/\textbf{63.29} & \textbf{70.68}/\textbf{61.87} & \textbf{70.67}/62.30 & \textbf{70.67}/\textbf{64.51} & \textbf{70.67}/\textbf{61.90} & \textbf{81.66} \\
\midrule
\multicolumn{11}{l}{\textbf{Metric: DS-AURC $\downarrow$ ResNet-18 - Trained on CIFAR-100 training set. Evaluated on CIFAR100 test set + Near/Far-OOD.}} \\
\midrule
\textbf{Training strategy} & MSP & OpenMax & MDS & Gram & ReAct & KLM & VIM & KNN & SIRC & Acc. \\ 
\midrule
Basic & 202.38/367.56 & 201.78/365.87 & 201.78/355.07 & 201.14/313.08 & 199.61/331.40 & 199.23/351.05 & 198.39/321.09 & 197.30/334.68 & 197.08/353.41 & 77.32 \\
Mixup & 199.15/357.54 & 194.40/\textbf{308.17} & 198.07/345.87 & 198.70/351.96 & 199.04/352.68 & 197.36/344.72 & 197.93/346.39 & 193.16/319.28 & 198.98/356.78 & 78.47 \\
RegMixup & 184.32/398.83 & 183.40/364.21 & 184.15/385.03 & 184.32/364.54 & 184.22/384.43 & 183.03/382.76 & 183.78/388.94 & 182.80/339.67 & 184.08/397.89 & 79.35 \\
AugMix & 212.08/341.92 & 211.71/336.07 & 211.78/335.16 & 209.38/328.47 & 204.51/301.55 & 211.52/335.61 & 206.50/318.61 & 206.42/321.53 & 202.19/327.13 & 76.98 \\
PixMix & 210.44/374.21 & 209.78/373.78 & 209.78/361.44 & 206.15/305.26 & 200.31/322.75 & 209.78/344.67 & 203.68/328.09 & 206.99/324.76 & 199.77/336.33 & 77.20 \\
CutMix & 257.72/426.99 & 248.02/350.24 & 246.43/374.27 & 256.78/406.69 & 256.78/417.24 & 246.96/390.84 & 252.25/402.06 & 216.56/312.48 & 256.78/425.78 & 77.81 \\
\midrule
SURE & 199.05/393.25 & 191.80/374.63 & 198.78/392.43 & 196.41/341.42 & 198.78/380.26 & 198.11/386.29 & 198.78/385.16 & 182.86/339.69 & 193.09/355.93 & 80.55 \\
- CSC & 188.42/367.92 & 187.78/335.35 & 187.78/347.37 & 185.59/351.69 & 180.23/299.64 & 184.64/335.85 & 187.35/302.61 & 187.48/298.00 & 185.31/361.24 & 80.36 \\
- CRL & 187.34/352.49 & 186.78/337.60 & 186.78/341.63 & 186.60/339.22 & 182.15/301.00 & 186.78/329.13 & 185.66/\textbf{274.06} & 185.43/270.08 & 183.83/343.65 & 80.68 \\
+ SWA$\rightarrow$EMA & 185.28/420.08 & 184.06/383.58 & 185.17/414.74 & 185.21/401.58 & 185.55/400.30 & 183.25/366.63 & 185.19/410.03 & 183.99/354.73 & 185.25/419.75 & 80.63 \\
+ ReBN & 183.61/355.63 & 182.78/343.26 & 182.78/346.65 & 183.01/340.78 & 177.42/299.51 & 181.67/337.65 & 183.23/305.04 & 182.78/273.98 & 180.41/345.06 & 80.54 \\
+ SAM$\rightarrow$F-SAM & 181.86/368.00 & 180.78/345.65 & 180.78/357.42 & 179.79/340.81 & 178.19/308.97 & 180.73/340.58 & 181.47/338.15 & 180.82/293.52 & 178.58/357.88 & 80.79 \\
+ RegPixMix (\textbf{SURE+}) & \textbf{173.45}/\textbf{314.04} & \textbf{173.13}/313.78 & \textbf{172.78}/\textbf{309.11} & \textbf{171.51}/\textbf{262.85} & \textbf{169.02}/\textbf{284.43} & \textbf{171.06}/\textbf{294.40} & \textbf{172.95}/293.50 & \textbf{173.45}/\textbf{252.85} & \textbf{170.28}/\textbf{293.09} & \textbf{81.66} \\ 
\bottomrule
\end{tabular}
}
\caption{\textbf{DS-F1 and DS-AURC results on CIFAR-100 (ResNet-18) with different training strategies}. We report DS-F1 and DS-AURC scores for near/far-OOD settings, along with ID accuracy. SURE+ consistently achieves the best DS-F1 and DS-AURC scores while maintaining or improving ID accuracy.}
\label{tab:table2_r18_ds}
\end{table*}

%% file: tables/table3_dinov3_ds.tex
\begin{table*}[ht]
\centering
\scalebox{.8}{
\begin{tabular}{l||cccccccccc} 
\toprule
\multicolumn{11}{l}{\textbf{Metric: DS-F1 $\uparrow$ DINOv3 ViT-L/16 - Fine-tuned on ImageNet-1K training set. Evaluated on ImageNet-1K test set + Near/Far-OOD.}} \\  
\midrule
\textbf{Training strategy} & \textbf{MSP} & \textbf{OpenMax} & \textbf{MDS} & \textbf{Gram} & \textbf{ReAct} & \textbf{KLM} & \textbf{VIM} & \textbf{KNN} & \textbf{SIRC} & \textbf{Acc.} \\ 
\midrule
Basic & 75.42/84.38 & 75.54/84.48 & 76.55/85.26 & 75.42/84.38 & 76.82/85.33 & 75.51/84.59 & 76.37/85.13 & 75.87/84.38 & 75.53/84.52 & 86.89 \\
Mixup & 75.09/84.03 & 75.09/84.06 & 76.54/85.18 & 75.09/84.02 & 75.96/85.04 & 75.21/84.25 & 75.50/84.52 & 75.10/84.25 & 75.09/84.05 & 87.01 \\
RegMixup & 75.01/83.84 & 75.01/83.90 & 76.12/85.10 & 75.01/83.84 & 76.62/85.27 & 75.28/84.30 & 75.35/84.37 & 75.01/84.01 & 75.05/84.12 & 87.07 \\
AugMix & 75.99/84.98 & 76.18/85.09 & \textbf{77.39}/86.08 & 75.99/84.98 & 77.22/85.99 & 76.17/85.29 & 76.85/85.81 & 76.26/85.43 & 76.07/85.11 & 87.72 \\
PixMix & 75.68/84.37 & 75.96/84.66 & 77.29/85.69 & 75.68/84.37 & 77.24/85.59 & 75.86/84.81 & 76.52/85.15 & 76.16/85.04 & 75.74/84.51 & 87.32 \\
CutMix & 74.97/84.00 & 75.06/84.12 & 76.34/85.33 & 74.96/84.00 & 76.33/85.26 & 75.31/84.35 & 75.49/84.70 & 74.97/84.33 & 75.05/84.37 & 87.12 \\
\midrule
SURE & 76.62/85.44 & 76.62/85.48 & 76.66/85.91 & 76.62/85.44 & 77.67/86.30 & 76.78/85.81 & 76.92/86.00 & 76.62/85.50 & 76.62/85.45 & 87.94 \\
- CSC & 76.53/85.68 & 76.53/85.68 & 77.12/86.35 & 76.53/85.68 & 77.87/86.71 & 76.68/85.91 & 76.83/86.16 & 76.53/85.74 & 76.55/85.68 & 87.95 \\
- CRL & 76.39/85.39 & 76.39/85.39 & 77.24/86.32 & 76.39/85.39 & 77.83/86.56 & 76.65/85.66 & 76.76/85.96 & 76.39/85.49 & 76.39/85.39 & 88.02 \\
+ SWA$\rightarrow$EMA & 76.35/84.29 & 76.37/84.39 & 76.98/86.10 & 76.45/84.41 & 77.67/86.46 & 76.64/85.05 & 76.34/85.05 & 76.23/84.86 & 76.23/85.13 & 88.05 \\
+ ReBN & 76.43/85.42 & 76.43/85.42 & 77.11/86.28 & 76.43/85.42 & 77.90/86.58 & 76.54/85.67 & 76.75/86.00 & 76.43/85.52 & 76.43/85.42 & 87.99 \\
+ SAM$\rightarrow$F-SAM & 76.41/85.40 & 76.41/85.40 & 77.16/86.30 & 76.41/85.40 & 77.88/86.56 & 76.55/85.67 & 76.75/85.97 & 76.41/85.48 & 76.41/85.40 & 88.04 \\
+ RegPixMix (\textbf{SURE+}) & \textbf{77.10}/\textbf{86.15} & \textbf{77.13}/\textbf{86.15} & 77.20/\textbf{86.72} & \textbf{77.10}/\textbf{86.15} & \textbf{78.01}/\textbf{87.13} & \textbf{77.35}/\textbf{86.46} & \textbf{77.40}/\textbf{86.78} & \textbf{77.10}/\textbf{86.15} & \textbf{77.10}/\textbf{86.15} & 88.49 \\
\midrule
\multicolumn{11}{l}{\textbf{Metric: DS-AURC $\downarrow$ DINOv3 ViT-L/16 - Fine-tuned on ImageNet-1K training set. Evaluated on ImageNet-1K test set + Near/Far-OOD.}} \\
\midrule
\textbf{Training strategy} & \textbf{MSP} & \textbf{OpenMax} & \textbf{MDS} & \textbf{Gram} & \textbf{ReAct} & \textbf{KLM} & \textbf{VIM} & \textbf{KNN} & \textbf{SIRC} & \textbf{Acc.} \\ 
\midrule
Basic & 168.33/48.07 & 167.28/47.41 & 154.64/40.43 & 164.43/47.78 & 146.42/39.29 & 168.33/46.90 & 153.82/41.95 & 156.56/47.92 & 165.24/46.84 & 86.89 \\
Mixup & 172.58/53.46 & 172.57/53.18 & 147.78/41.47 & 171.81/53.30 & 162.90/47.33 & 171.63/51.73 & 168.68/50.25 & 170.21/48.53 & 172.18/53.30 & 87.01 \\
RegMixup & 169.43/52.81 & 169.43/51.96 & 149.98/41.17 & 165.30/52.51 & 140.66/40.84 & 168.72/50.30 & 162.51/47.54 & 167.51/48.42 & 162.12/50.23 & 87.07 \\
AugMix & 163.74/44.63 & 162.05/43.53 & 140.45/35.67 & 161.93/44.57 & 142.35/36.25 & 163.05/42.72 & 149.90/38.38 & 154.66/40.47 & 161.44/43.70 & 87.72 \\
PixMix & 167.34/48.95 & 165.81/47.27 & 139.82/37.35 & 164.22/48.59 & 141.90/38.49 & 167.34/46.87 & 153.69/42.43 & 156.34/42.54 & 163.76/47.61 & 87.32 \\
CutMix & 169.78/51.94 & 169.52/51.02 & 150.77/39.41 & 165.73/51.56 & 149.11/40.53 & 167.82/49.93 & 164.51/45.67 & 167.91/47.23 & 164.49/49.03 & 87.12 \\
\midrule
SURE & 157.07/42.13 & 157.07/41.51 & 155.56/36.52 & 155.96/42.00 & 155.21/36.92 & 156.04/39.39 & 151.03/38.61 & 156.02/39.20 & 154.99/41.56 & 87.94 \\
- CSC & 156.47/39.70 & 155.77/39.70 & 157.18/34.66 & 153.58/39.56 & 153.69/32.72 & 155.00/38.09 & 151.38/36.56 & 155.66/38.89 & \textbf{151.44}/39.29 & 87.95 \\
- CRL & 157.30/43.27 & 157.30/43.12 & \textbf{154.20}/35.27 & 155.31/43.18 & 154.32/35.01 & 157.18/41.12 & 150.96/39.15 & 157.05/41.53 & 153.39/42.92 & 88.02 \\
+ SWA$\rightarrow$EMA & 157.56/43.86 & 156.83/43.86 & 157.26/36.17 & 155.68/42.99 & 151.78/35.44 & 156.69/41.53 & 150.60/38.81 & 157.56/42.86 & 152.83/43.49 & 88.05 \\
+ ReBN & 156.98/42.36 & 155.82/42.36 & 156.46/35.17 & \textbf{154.21}/42.27 & 149.38/34.14 & 155.84/40.43 & 150.90/38.21 & 155.96/41.00 & 152.07/42.02 & 87.99 \\
+ SAM$\rightarrow$F-SAM & 156.77/42.75 & 156.04/42.75 & 155.45/35.24 & 154.66/42.63 & 148.94/34.40 & 156.20/40.76 & 150.84/38.56 & 156.25/41.23 & 152.44/42.41 & 88.04 \\
+ RegPixMix (\textbf{SURE+}) & \textbf{156.00}/\textbf{38.07} & \textbf{155.00}/\textbf{38.07} & 154.59/\textbf{34.39} & 154.56/\textbf{38.07} & \textbf{144.71}/\textbf{31.56} & \textbf{156.00}/\textbf{36.14} & \textbf{150.63}/\textbf{33.76} & \textbf{150.12}/\textbf{38.07} & 152.00/\textbf{37.73} & 88.49 \\
\bottomrule
\end{tabular}
}
\caption{\textbf{DS-F1 and DS-AURC results on ImageNet-1K (DINOv3 ViT-L/16) with different fine-tuning strategies}. We report DS-F1  and DS-AURC scores for near/far-OOD settings, along with ID accuracy. SURE+ consistently achieves the best DS-F1 and DS-AURC scores while maintaining or improving ID accuracy.}
\label{tab:table3_dinov3_ds}
\end{table*}

%% file: sections/7_futurework.tex
\section{Challenges and Future Directions}

A significant frontier in enhancing classifier reliability involves the integration of training-based OOD methods that utilize auxiliary data, such as Outlier Exposure (OE)~\cite{hendrycks2019outlier} or OpenMix~\cite{zhu2023openmix}. While post-hoc methods rely on signals from pre-trained models, training-based approaches explicitly modify the learning process to make models more sensitive to OOD inputs. Future research should explore how incorporating such auxiliary outliers can be optimized within a double-scoring framework, potentially leading to more distinct decision boundaries between ID and OOD samples. For instance, while current recipes like SURE+ utilize augmentation-based regularization such as RegPixMix and RegMixup to improve robustness without external OOD data, the explicit exposure to simulated outliers during training remains a promising avenue for further narrowing the reliability gap.

Future work must also address the persistent performance bottleneck in Near-OOD scenarios. Experimental evidence indicates that while advanced OOD detection methods provide substantial gains for Far-OOD shifts, they offer only marginal benefits when OOD samples are visually similar to ID data. This suggests that current post-hoc scoring functions and training objectives may struggle to capture the fine-grained semantic differences required for Near-OOD discrimination. Consequently, developing more discriminative feature representations or specialized loss functions that target these challenging distribution shifts remains a critical priority for deploying trustworthy AI in realistic, safety-critical environments. In addition, threshold selection remains an underexplored yet practically critical component of real-world deployment. Current practice typically relies on manually tuned thresholds calibrated on held-out validation data, which may still fail to generalize under distribution shift. One promising direction is to leverage powerful generative models to synthesize unified and controllable data distributions for calibration. Such synthetic data could serve as a proxy for potential shift scenarios, enabling more stable and reliable decision boundary estimation without extensive manual tuning.

%% file: sections/8_conclusion.tex
\section{Conclusion}
\label{sec:conclusion}
In this work, we present a unified perspective on classifier reliability, emphasizing the complementary nature of OOD detection and failure prediction. We show that evaluating these aspects in isolation is suboptimal and propose two complementary metrics, DS-F1 and DS-AURC, to jointly assess a model’s ability to identify OOD inputs and anticipate its own errors. Extensive experiments on the OpenOOD benchmark validate the effectiveness of our unified evaluation framework, demonstrating its ability to distinguish classifiers that are truly robust and trustworthy. We further introduce SURE+, an improved, reliable classifier that integrates recent advances in both OOD detection and failure prediction. Empirical results confirm that SURE+ achieves better reliability across different post-hoc scores, highlighting the practical value of our approach. We believe that our framework and metrics offer a principled foundation for future research on trustworthy AI systems.

%% file: sections/9_acknowledge.tex

%% file: supp.tex
\appendices

\section{Details on proposed evaluation metrics}
\subsection{Implementation for DS-F1 and DS-AURC}\label{section:supp_implementation}
Algorithm~\ref{alg:ds-f1-calc} summarizes the computation procedure for DS-F1 defined in Eq.~\ref{eq:dual-f1}. 
Given a model $m$ and corresponding scoring functions $s_{\text{ID}}$ and $s_{\text{OOD}}$, the algorithm evaluates the F1 score over a grid of ID and OOD thresholds. 
For each threshold pair $(\tau_{\text{ID}}, \tau_{\text{OOD}})$, it computes precision and recall based on correct ID classifications. 
The final DS-F1 score is the maximum F1 obtained over all threshold combinations. 

\begin{algorithm}
\caption{Computation of DS-F1 on Evaluation Set}
\label{alg:ds-f1-calc}
\begin{algorithmic}[1]
\Require
Model $m$; ID and OOD scoring functions $s_{\text{ID}}, s_{\text{OOD}}$; 
evaluation set $\mathcal{D} = \mathcal{D}_{\text{ID}} \cup \mathcal{D}_{\text{OOD}}$; 
number of thresholds $T_{\text{grid}}$
\Ensure
DS-F1 score

\State Compute scores $s_{\text{ID}}(x_i)$ and $s_{\text{OOD}}(x_i)$ for all $x_i \in \mathcal{D}$
\State Initialize $\text{DS-F1} \gets 0$

\State Obtain quantile thresholds
$\{\tau_{\text{ID},t}\}_{t=1}^{T_{\text{grid}}}$ and
$\{\tau_{\text{OOD},t}\}_{t=1}^{T_{\text{grid}}}$

\ForAll{$\tau_{\text{ID}} \in \{\tau_{\text{ID},t}\}$}
    \ForAll{$\tau_{\text{OOD}} \in \{\tau_{\text{OOD},t}\}$}

        \State Construct acceptance set
        \[
        \mathcal{A} = \{\, i \mid
        s_{\text{ID}}(x_i) \ge \tau_{\text{ID}}
        \ \wedge \
        s_{\text{OOD}}(x_i) \ge \tau_{\text{OOD}} \,\}
        \]

        \State $A_{\text{ID}} \gets \mathcal{A} \cap \mathcal{D}_{\text{ID}}$

        \If{$|\mathcal{A}| = 0$}
            \State $\text{F1} \gets 0$
        \Else
            \State Compute ID error count
            \[
            \epsilon_{\text{ID}} =
            \sum_{i \in A_{\text{ID}}}
            \mathbb{I}\!\left(m(x_i) \neq y_i\right)
            \]

            \State Compute precision and recall

            \State $\text{precision} \leftarrow (|A_{\text{ID}}| - \epsilon_{\text{ID}}) / |\mathcal{A}|$
            \State $\text{recall} \leftarrow (|A_{\text{ID}}| - \epsilon_{\text{ID}}) / |\mathcal{D}_{\text{ID}}|$
            
            \State $\text{F1} \leftarrow 2 \cdot \text{precision} \cdot \text{recall} / (\text{precision} + \text{recall})$
        \EndIf

        \State $\text{DS-F1} \gets \max(\text{DS-F1}, \text{F1})$
    \EndFor
\EndFor

\State \Return $\text{DS-F1}$
\end{algorithmic}
\end{algorithm}

Algorithm~\ref{alg:dual-aurc-calc} summarizes the computation procedure for DS-AURC.
Following Eqs.~\ref{eq:coverage}, \ref{eq:selrisk-dual}, and \ref{eq:dual-aurc-prob}, we first discretize coverage into $K$ intervals and enumerate thresholds to obtain a finite set of $(u, \text{risk})$ pairs. 
For each threshold (or threshold pair in the double scoring case), the algorithm calculates the accepted sample set, computes the corresponding selective risk, and records the coverage level. 
Finally, the risk values are aggregated within each coverage interval, and the minimization is applied for the aggregation when we consider DS-AURC. Finally, we integrate numerically to obtain the final AURC or DS-AURC score. 

\begin{algorithm}[h]
\caption{Computation of (DS-)AURC on Evaluation Set}
\label{alg:dual-aurc-calc}
\begin{algorithmic}[1]
\Require
Model $m$; scoring functions $s_{\text{ID}}, s_{\text{OOD}}$ (single or double);
evaluation set $\mathcal{D}=\mathcal{D}_{\text{ID}}\cup\mathcal{D}_{\text{OOD}}$;
number of thresholds $T_{\text{single}}$ or $T_{\text{grid}}$;
number of coverage bins $K$
\Ensure
AURC or DS-AURC score

\State Compute scores for all samples
\State Initialize set $\mathcal{P} \gets \emptyset$

\If{single scoring}
    \State Obtain quantile thresholds $\{\tau_t\}_{t=1}^{T_{\text{single}}}$
    \ForAll{$\tau \in \{\tau_t\}$}
        \State $\mathcal{A} \gets \{ i \mid s(x_i) \ge \tau \}$
        \State $A_{\text{ID}} \gets \mathcal{A} \cap \mathcal{D}_{\text{ID}}$
        \State $A_{\text{OOD}} \gets \mathcal{A} \cap \mathcal{D}_{\text{OOD}}$

        \If{$|\mathcal{A}| = 0$}
            \State $\text{risk} \gets 0$
        \Else
            \State $\epsilon_{\text{ID}} \gets
            \sum_{i \in A_{\text{ID}}} \mathbb{I}(m(x_i) \neq y_i)$
            \State $\text{risk} \gets
            (\epsilon_{\text{ID}} + |A_{\text{OOD}}|) / |\mathcal{A}|$
        \EndIf

        \State $u \gets |A_{\text{ID}}| / |\mathcal{D}_{\text{ID}}|$
        \State Append $(u,\text{risk})$ to $\mathcal{P}$
    \EndFor
\Else
    \State Obtain quantile thresholds
    $\{\tau_{\text{ID},t}\}$ and $\{\tau_{\text{OOD},t}\}$
    \ForAll{$(\tau_{\text{ID}},\tau_{\text{OOD}})$}
        \State $\mathcal{A} \gets
        \{ i \mid s_{\text{ID}}(x_i)\!\ge\!\tau_{\text{ID}}
        \wedge
        s_{\text{OOD}}(x_i)\!\ge\!\tau_{\text{OOD}} \}$

        \State $A_{\text{ID}} \gets \mathcal{A} \cap \mathcal{D}_{\text{ID}}$
        \State $A_{\text{OOD}} \gets \mathcal{A} \cap \mathcal{D}_{\text{OOD}}$

        \If{$|\mathcal{A}| = 0$}
            \State $\text{risk} \gets 0$
        \Else
            \State $\epsilon_{\text{ID}} \gets
            \sum_{i \in A_{\text{ID}}} \mathbb{I}(m(x_i) \neq y_i)$
            \State $\text{risk} \gets
            (\epsilon_{\text{ID}} + |A_{\text{OOD}}|) / |\mathcal{A}|$
        \EndIf

        \State $u \gets |A_{\text{ID}}| / |\mathcal{D}_{\text{ID}}|$
        \State Append $(u,\text{risk})$ to $\mathcal{P}$
    \EndFor
\EndIf

\State Partition $[0,1]$ into $K$ coverage intervals
\State Aggregate risks per interval using $\phi$ (e.g., minimum)
\State \Return area under the risk--coverage curve
\end{algorithmic}
\end{algorithm}

\input{tables/table4_app_sep}

\subsection{Consistency with single-scoring metrics}\label{section:supp_proof}
We provide formal arguments showing that DS-F1 and DS-AURC strictly generalize their standard counterparts. The key observations rely only on the definitions of minimum and maximum over sets.

\begin{lemma}[Property of minima]
For any function $g: \mathcal{X} \to \mathbb{R}$ and any $x_0 \in \mathcal{X}$, it holds that
\[
\min_{x \in \mathcal{X}} g(x) \;\le\; g(x_0).
\]
\end{lemma}
\begin{lemma}[Property of maxima under set inclusion]
Let $h: \mathcal{X} \to \mathbb{R}$ and $S_1 \subseteq S_2 \subseteq \mathcal{X}$. Then
\[
\max_{x \in S_1} h(x) \;\le\; \max_{x \in S_2} h(x).
\]
\end{lemma}
Both lemmas are direct consequences of the definitions of minimum and maximum.

\textbf{DS-F1 dominates F1.}
Let $f(\tau_\text{OOD}, \tau_\text{ID})$ denote the ds-F1 score obtained under thresholds $(\tau_\text{OOD}, \tau_\text{ID})$. Then, we can simplify Eq.~\ref{eq:dual-f1} as:

\begin{equation*}
    \text{DS-F1} = \max_{\tau_\text{OOD}, \tau_\text{ID}} f(\tau_\text{OOD}, \tau_\text{ID})
\end{equation*}

If we consider a system with both ID and OOD samples evaluating with the classic F1 score and the previous pipeline, the results become:

\begin{equation*}
    \text{F1} = \max_{\tau_\text{ID}} f(\tau'_\text{OOD}, \tau_\text{ID}) 
\end{equation*}
where $\tau'_\text{OOD}$ is a fixed optimal threshold obtained by a certain post-hoc method. Then we have:

\begin{equation*}
    \text{F1} \leq \text{DS-F1}
\end{equation*}
and this inequality will also stand when we have a fixed $\tau_{\text{ID}}$.

\begin{proof}
Define the feasible sets
\[
\begin{aligned}
S_1 &= \{ (\tau'_{\text{OOD}}, \tau_{\text{ID}})
        \mid \tau_{\text{ID}} \in [0,1] \}, \\
S_2 &= \{ (\tau_{\text{OOD}}, \tau_{\text{ID}})
        \mid \tau_{\text{OOD}}, \tau_{\text{ID}} \in [0,1] \}.
\end{aligned}
\]

Clearly $S_1 \subseteq S_2$. By Lemma 2,
\[
\max_{(\tau_\text{OOD},\tau_\text{ID})\in S_1} f(\tau_\text{OOD},\tau_\text{ID})
\le
\max_{(\tau_\text{OOD},\tau_\text{ID})\in S_2} f(\tau_\text{OOD},\tau_\text{ID}).
\]
This is equivalent to $\text{F1} \le \text{DS-F1}$.
\end{proof}

\textbf{DS-AURC dominates AURC.}
When we calculate the regular AURC to filter both misclassified ID samples and accepted OOD samples from the whole acceptance set defined in Eq.~\ref{eq:setA}, we can consider that OOD filtering is fixed, the acceptance set is restricted to $\mathcal{A}(\tau'_{\text{OOD}}, \tau_{\text{ID}})$, where $\tau'_{\text{OOD}}$ is a fixed threshold for OOD rejection. Accordingly, we then can rewrite Eq.~\ref{eq:selrisk-dual} as
\[
\begin{aligned}
\text{SR}(\tau_{\text{ID}})
&= \frac{
    \sum_{i \in \mathcal{D}_{\text{ID}}}
    Z_i \cdot \mathbb{I}
    \bigl(i \in \mathcal{A}(\tau'_{\text{OOD}}, \tau_{\text{ID}})\bigr)
}{ 
    \bigl|
    \mathcal{A}(\tau'_{\text{OOD}}, \tau_{\text{ID}})
    \bigr|
}
\\[4pt]
&\quad +
\frac{
    \bigl|
    \mathcal{A}(\tau'_{\text{OOD}}, \tau_{\text{ID}})
    \cap
    \mathcal{D}_{\text{OOD}}
    \bigr|
}{
    \bigl|
    \mathcal{A}(\tau'_{\text{OOD}}, \tau_{\text{ID}})
    \bigr|
}.
\end{aligned}
\]

Since the coverage in Eq.~\ref{eq:coverage} is defined over all admissible acceptance sets, we have
\begin{equation*} \Bigl\{\,x \;\big|\; x=\text{SR}(\tau_{\text{ID}})\,\Bigr\} \;\subseteq\; \Bigl\{\,x \;\big|\; x=\text{SR}(\tau_{\text{OOD}}, \tau_{\text{ID}})\,\Bigr\}. \end{equation*}
That is, the set of risks attainable under standard AURC is a strict subset of those attainable under DS-AURC. Consequently, if we use the minimum selective risk under a given coverage as the aggregation operator (cf. Eq.~\ref{eq:dual-aurc-prob}), DS-AURC necessarily achieves lower risk:
\begin{equation*}
    \text{DS-AURC}(m, s_\text{OOD}, s_\text{ID})
    \;\le\;
    \text{AURC}(m, s).
\end{equation*}

\begin{proof}
Given a fix $\tau_\text{OOD}$, by Lemma 1 we have,
\[
\min_{\tau_\text{ID}} \text{SR}(\tau_\text{OOD}, \tau_\text{ID})
\;\;\geq\;\;
\min_{\tau_\text{OOD}, \tau_\text{ID}} \text{SR}(\tau_\text{OOD}, \tau_\text{ID}).
\]
Taking expectation over coverage, it follows that
\[
\text{DS-AURC}(m, s_\text{OOD}, s_\text{ID})
\;\;\leq\;\;
\text{AURC}(m, s).
\]
\end{proof}

The inequalities above are tight when the fixed thresholds 
$(\tau'_\text{OOD},\tau_\text{ID})$ or $(\tau_\text{OOD},\tau'_\text{ID})$ 
coincide with the global optima of the joint problem. Otherwise, the inequalities are strict. This shows that:
\begin{itemize}
    \item DS-F1 never performs worse than the standard F1, and can be strictly higher if ID and OOD thresholds interact non-trivially.
    \item DS-AURC never performs worse (i.e., never larger) than AURC, and can be strictly smaller in practice.
\end{itemize}
Thus, double scoring is a strict generalization of the traditional single-scoring evaluation pipeline. 
In particular, by jointly considering two scores during evaluation, the system can simultaneously account for correct ID predictions, misclassified ID samples, and OOD samples, leading to both lower selective risk and higher F1 score. This demonstrates the practical advantage and improved expressiveness of the double scoring framework.

\section{Performance of SURE+ on Separate Tasks}
\label{app_sec:surep}

In the main paper, we primarily evaluate different training strategies using the proposed metrics, i.e., DS-F1 and DS-AURC. Here, we further provide a more detailed analysis by reporting the conventional metrics separately for OOD detection (AUROC, FPR), ID failure prediction (AURC), and Accuracy. As shown in Table~\ref{tab4:detail_performance}, we observe that the model reliability benefits more from the combination of different training strategies. Specifically, our proposed \textbf{SURE+} consistently achieves competitive or superior results across both tasks, demonstrating its effectiveness in improving model reliability from both the OOD Detection and failure prediction perspectives. 

\section{Choice of ID Confidence Score}
\label{app_choice:id_score}
In the main paper, we use MSP as the default ID confidence score when computing DS-F1 and DS-AURC. To study the impact of ID confidence choice, Table~\ref{tab:id_score_ablation} reports results using different post-hoc scoring methods for ID and OOD samples. In addition to MSP, we include ReAct and VIM 
since they are widely adopted and representative post-hoc methods for OOD detection.

As shown in the table, MSP exhibits stable performance when used as either the ID or OOD score. In contrast, although ReAct and VIM perform well for OOD detection, their performance degrades when used as ID confidence scores, with VIM showing the most severe degradation. This indicates that post-hoc methods optimized for OOD detection may not reliably capture fine-grained confidence differences among ID samples. Therefore, we adopt MSP as the ID score in our main experiments due to its stability in mixed ID and OOD settings. 

\input{tables/table5_app_idscore}

\input{tables/table6_openood}

\section{Full Results based on OpenOOD Settings}\label{section:supp_openood}
We present full results based on OpenOOD settings~\cite{yang2022openood,zhang2023openood} in this section. Unlike the double-scoring selective classification setup in the main paper experiments, where the scoring function was required to separate correctly classified ID samples from both misclassified ID and OOD samples, the OpenOOD setting follows the standard OOD detection protocol, focusing solely on distinguishing ID from ID+OOD data during evaluation. This evaluation is thus simpler but widely adopted in the OOD detection literature.
Presenting these results aligns our evaluation with the established OpenOOD benchmark, verifies the correctness of our implementation, and complements the more challenging double-scoring selective classification experiments reported in the main paper.

Concretely, we report results on ResNet-18~\cite{he2016deep} trained on CIFAR-100, benchmarking a variety of post-hoc OOD detection methods within the unified OpenOOD framework. While we use the same model as in the main paper, the evaluation setting here differs: we adopt AUROC, FPR@95, and AUPR as metrics, following the OpenOOD protocol. Accordingly, the objective is also aligned with OpenOOD, focusing solely on distinguishing ID data from OOD data.

As shown in Table~\ref{tab:supp_openoodr18_color}, we observe that the ranking of the results generally aligns with the results in the more recent OpenOOD 1.5 benchmark~\cite{zhang2023openood}, which demonstrates the reliability of our implementation.

\section{Implementation Details of SURE+}\label{section:sureplus_supp}
SURE+ builds upon SURE~\cite{li2024sure} with several modifications to its components. Specifically, the original SURE loss function consists of a standard cross-entropy (CE) loss, a RegMixup loss~\cite{pinto2022using}, and a CRL loss~\cite{dong2017class}, with a cosine-similarity classifier (CSC) as the prediction head. As described in the main paper, SURE+ removes the CRL loss and replaces the CSC with the linear classifier and introduces an additional RegPixMix loss based on PixMix~\cite{hendrycks2022pixmix}. Furthermore, SURE+ also applies F-SAM~\cite{li2024friendly} and EMA to replace FMFP optimizations~\cite{zhu2023revisiting}, i.e. SAM~\cite{foret2020sharpness} and SWA~\cite{izmailov2018averaging}. We list the implementation details as follows.

\textbf{RegMixup.} 
Given a training sample $(x,y)$ and a randomly shuffled pair $(x',y')$ from the same batch, we sample a mixing coefficient $\kappa \sim \mathrm{Beta}(\alpha,\alpha)$ and construct a mixed input
\[
\tilde{x} = \kappa x + (1-\kappa)x',
\]
The RegMixup loss is defined as a weighted sum of classification losses on the original and shuffled labels:
\[
\mathcal{L}_{\mathrm{RegMixup}} = \kappa \,\mathcal{L}_\text{CE}(m(\tilde{x}),y) + (1-\kappa)\,\mathcal{L}_\text{CE}(m(\tilde{x}),y').
\]
where $m(\cdot)$ denotes the network align with the notation in the main paper and $\mathcal{L}_\text{CE}(\cdot,\cdot)$ is the cross-entropy loss. This formulation is equivalent to training on a soft label $\kappa y + (1-\kappa)y'$, and serves as a regularization term that improves robustness. Specifically, we follow the original RegMixup setting and adopt $\alpha = 10$ in the experiments.

\textbf{RegPixMix.}
PixMix augments each input by stochastically mixing it with an auxiliary image $x_{\mathrm{mix}}$ sampled from a set of texture images~\cite{hendrycks2022pixmix}. Starting from the original image $x$, we optionally apply a random augmentation, and then repeatedly (1–4 times) select either another augmentation of $x$ or $x_{\mathrm{mix}}$ to serve as the mixing partner. A mixing operator $M$ is then drawn from a predefined set of transformations $\mathcal{M}$, and the final mixed image is constructed as
\[
x_{\mathrm{pixmix}} = M(x, x_{\mathrm{mix}}),
\]
The corresponding loss is simply
\[
\mathcal{L}_{\mathrm{RegPixMix}} = \mathcal{L}_\text{CE}(m(x_{\mathrm{pixmix}}),y).
\]

\textbf{F-SAM.}
To further enhance generalization, we employ \emph{Friendly Sharpness-Aware Minimization (F-SAM)}~\cite{li2024friendly} as an optimization strategy. 
F-SAM perturbs the model parameters along the normalized gradient direction
\[
\epsilon_{\mathrm{FSAM}} = \rho_f \cdot \frac{\nabla_\theta \mathcal{L}}{\|\nabla_\theta \mathcal{L}\|_2},
\]
and minimizes a convex combination of the standard and perturbed losses:
\[
\mathcal{L}_{\mathrm{FSAM}} = \gamma \,\mathcal{L}(\theta) + (1-\gamma)\,\mathcal{L}(\theta + \epsilon_{\mathrm{FSAM}}).
\]
This formulation encourages convergence towards flatter minima while mitigating the over-regularization observed in standard SAM~\cite{foret2020sharpness}.

\textbf{EMA.}
We further stabilize training and evaluation with two complementary techniques. 
First, we maintain an \emph{Exponential Moving Average (EMA)} of the model parameters:
\[
\theta_{\mathrm{EMA}} \leftarrow \tau \theta_{\mathrm{EMA}} + (1-\tau)\theta,
\]
where $\tau$ is a momentum coefficient. The EMA parameters $\theta_{\mathrm{EMA}}$ are used for evaluation, providing a smoother trajectory and reducing the variance of predictions. 
Second, following standard practice, We apply \emph{Re-normalized Batch Normalization} (ReBN) after the final training epoch. 
Specifically, we update the batch normalization statistics by running a single forward pass over the entire training set without gradient updates. 
EMA and ReBN jointly improve the stability of learned representations under distribution shifts, leading to more reliable calibration and robustness.

\textbf{Overall Objective.}
The overall training objective of SURE+ is defined as
\[
\mathcal{L}_{\text{SURE+}} = \mathcal{L}_{\mathrm{CE}}(x,y) + \mathcal{L}_{\mathrm{RegMixup}} +  \mathcal{L}_{\mathrm{RegPixMix}}.
\]
The optimization is further regularized by F-SAM, and evaluation is performed with EMA, leading to improved robustness, generalization, and calibration.

%% file: tables/table4_app_sep.tex
\begin{table*}[t]
\centering
\scalebox{1.0}{
\begin{tabular}{l||c|c|c} 
\toprule

\multirow{2}{*}{\textbf{Training strategy}} 
& \multicolumn{1}{c|}{\textbf{OOD Detection}} 
& \textbf{Failure Prediction} 
& \multirow{2}{*}{\textbf{Acc.} $\uparrow$} \\ 
\cmidrule(lr){2-2} \cmidrule(lr){3-3} 
& \textbf{AUROC} $\uparrow$ 
& \textbf{AURC} $\downarrow$ & \\ 
\midrule
\multicolumn{4}{l}{\textit{\textbf{ResNet-18 - Trained on CIFAR-100}}} \\ 
\midrule
Basic                & 80.82{\scriptsize $\pm$0.05} / 79.31{\scriptsize $\pm$0.99}  & 61.64{\scriptsize $\pm$1.46} & 77.32 \\
Mixup               & 80.54{\scriptsize $\pm$0.02} / 79.25{\scriptsize $\pm$0.42}  & 57.15{\scriptsize $\pm$0.83} & 78.47 \\
RegMixup            & \textbf{\textcolor{blue1}{81.11{\scriptsize $\pm$0.44}}} / 76.69{\scriptsize $\pm$0.91} & 51.71{\scriptsize $\pm$1.27} & 79.35 \\
AugMix              & 79.66{\scriptsize $\pm$0.02} / \textbf{\textcolor{blue1}{79.47{\scriptsize $\pm$0.22}}} & 60.45{\scriptsize $\pm$0.18} & 76.98 \\
PixMix              & 79.88{\scriptsize $\pm$0.54} / \textbf{\textcolor{blue3}{81.02{\scriptsize $\pm$1.00}}} & 60.02{\scriptsize $\pm$2.03} & 77.20 \\
CutMix              & 77.66{\scriptsize $\pm$1.49} / 76.49{\scriptsize $\pm$1.87} & 68.68{\scriptsize $\pm$7.06}  & 77.81 \\

\midrule
SURE                & 79.23{\scriptsize $\pm$0.22} / 74.24{\scriptsize $\pm$0.46} & 44.59{\scriptsize $\pm$1.18} & 80.55 \\
- CSC               & 80.50{\scriptsize $\pm$0.09} / 76.42{\scriptsize $\pm$1.32} & 44.66{\scriptsize $\pm$0.76} & 80.36 \\
- CRL               & 80.51{\scriptsize $\pm$0.16} / 77.40{\scriptsize $\pm$0.34} & \textbf{\textcolor{blue1}{43.68{\scriptsize $\pm$1.34}}} & \textbf{\textcolor{blue1}{80.68}} \\
+ SWA $\rightarrow$ EMA  & 80.91{\scriptsize $\pm$0.66} / 77.69{\scriptsize $\pm$0.69} & 43.72{\scriptsize $\pm$2.38} & 80.54 \\
+ SAM $\rightarrow$ FSAM & \textbf{\textcolor{blue3}{81.14{\scriptsize $\pm$0.29}}} / 77.16{\scriptsize $\pm$0.55}  & \textbf{\textcolor{blue3}{43.29{\scriptsize $\pm$0.97}}} & \textbf{\textcolor{blue3}{80.79}} \\
+ RegPixMix(\textbf{SURE+}) & \textbf{\textcolor{blue5}{82.06{\scriptsize $\pm$0.60}}} / \textbf{\textcolor{blue5}{82.95{\scriptsize $\pm$0.49}}}  & \textbf{\textcolor{blue5}{39.44{\scriptsize $\pm$1.99}}} & \textbf{\textcolor{blue5}{81.66}} \\
\midrule
\multicolumn{4}{l}{\textit{\textbf{DINOv3 ViT-L/16 - Fine-tuned on ImageNet-1K}}} \\ 
\midrule
Basic & 81.69 / 90.63 & 28.28 &  86.89 \\
Mixup & 80.54 / 89.44 & 29.91 & 87.01 \\
RegMixup & 80.54 / 88.91 & 28.26 &  87.07 \\
AugMix & 81.77 / 91.04 & 25.84 &  87.72 \\
PixMix & 81.38 / 90.04 & 27.13  &  87.32 \\
CutMix & 80.33 / 89.07 & 27.93 &  87.12 \\
\midrule
SURE & \textbf{\textcolor{blue1}{82.70}} / \textbf{\textcolor{blue1}{91.50}} & \textbf{\textcolor{blue1}{24.30}} &  87.94 \\
- CSC & 82.27 / 91.25 & 24.41 &   \textbf{\textcolor{blue3}{88.04}} \\
- CRL & 82.26 / 91.13 & 24.51 &  \textbf{\textcolor{blue1}{88.02}}  \\
+ SWA $\rightarrow$ EMA & 82.29 / 91.34 & 24.32 &  87.99 \\
+ SAM $\rightarrow$ FSAM & \textbf{\textcolor{blue3}{82.70}} / \textbf{\textcolor{blue3}{92.37}} & \textbf{\textcolor{blue3}{24.16}} & 87.95 \\
+ RegPixMix(\textbf{SURE+}) & \textbf{\textcolor{blue5}{83.05}} / \textbf{\textcolor{blue5}{92.56}} & \textbf{\textcolor{blue5}{22.54}} & \textbf{\textcolor{blue5}{88.49}} \\
\bottomrule
\end{tabular}
}
\caption{Comparison of \textbf{OOD detection} performance (AUROC), \textbf{Failure Prediction} performance (AURC), and \textbf{in-distribution} (ID) classification accuracy (Acc) on CIFAR-100 and ImageNet-1K under different training strategies. All experiments are conducted using MSP as the scoring function. Results on CIFAR-100 are averaged over three runs, with both the mean and standard deviation reported. The top-3 methods for each metric are highlighted using a color gradient from \textbf{\textcolor{blue1}{light blue}} to \textbf{\textcolor{blue5}{dark blue}}.}
\label{tab4:detail_performance}
\end{table*}

%% file: tables/table5_app_idscore.tex
\begin{table}[t]
\centering
\small

\scalebox{.90}{
\begin{tabular}{l||c|c|c}
\toprule
\multirow{2}{*}{OOD Score} & \multicolumn{3}{c}{ID Score} \\
\cmidrule(lr){2-4}
 & MSP~\cite{hendrycks2016baseline} & ReAct~\cite{sun2021react} & VIM~\cite{wang2022vim} \\
\midrule
\multicolumn{4}{l}{\textit{\textbf{DS-F1} $\uparrow$ }} \\
\midrule
MSP~\cite{hendrycks2016baseline}   & 67.42 / 57.03 & 67.42 / 57.03 & 47.25 / 46.61 \\
ReAct~\cite{sun2021react} & 67.43 / 58.44 & 66.30 / 58.02 & 47.21 / 46.60 \\
VIM~\cite{wang2022vim}   & \textbf{67.46 / 58.94} & 59.94 / 55.03 & 47.17 / 46.56 \\
\midrule
\multicolumn{4}{l}{\textit{\textbf{DS-AURC} $\downarrow$ }} \\
\midrule
MSP~\cite{hendrycks2016baseline}   & 202.38 / 367.56 & 202.56 / 368.06 & 673.00 / 684.09 \\
ReAct~\cite{sun2021react} & 199.61 / 331.40 & 213.45 / 335.92 & 673.20 / 678.29 \\
VIM~\cite{wang2022vim}   & \textbf{198.39 / 321.09} & 270.88 / 359.96 & 679.89 / 678.12 \\
\bottomrule
\end{tabular}
}
\caption{Effect of different post-hoc scoring methods used as ID confidence scores on CIFAR-100 with ResNet-18. Rows denote the post-hoc method used for scoring OOD samples, while columns denote the method used for scoring ID samples. Results are reported for Near/Far-OOD settings.}
\label{tab:id_score_ablation}
\end{table}

%% file: tables/table6_openood.tex
\begin{table*}[t]
\centering
\scalebox{1.}{
\begin{tabular}{l||c|c|c} 
\toprule
\multirow{2}{*}{\textbf{Method}} & \multicolumn{3}{l}{\begin{tabular}[c]{@{}l@{}}\textbf{\textbf{Model}}: ResNet-18~\textbf{\textbf{Training strategy}}: Basic\\\textbf{\textbf{\textbf{\textbf{Training/Validation set}}}}: CIFAR-100 training set\\\textbf{\textbf{Evaluation set}}: CIFAR-100 test set + Near/Far OOD sets\end{tabular}} \\ 
\cmidrule{2-4}
 & \textbf{AUROC $\uparrow$} & \textbf{FPR@95$\downarrow$} & \textbf{AUPR$\uparrow$} \\ 
\midrule
MSP~\cite{hendrycks2016baseline} & \textbf{\textcolor{blue1}{80.82}}{\scriptsize $\pm$0.05} / 79.31{\scriptsize $\pm$0.99} & \textbf{\textcolor{blue5}{54.34}}{\scriptsize $\pm$0.60} / 56.64{\scriptsize $\pm$2.33} & \textbf{\textcolor{blue1}{83.99}}{\scriptsize $\pm$2.33} / 66.73{\scriptsize $\pm$1.71} \\
OpenMax~\cite{bendale2016towards} & 72.34{\scriptsize $\pm$0.37} / 73.10{\scriptsize $\pm$0.81} & 55.56{\scriptsize $\pm$0.21} / 59.83{\scriptsize $\pm$1.63} & 80.10{\scriptsize $\pm$1.63} / 62.92{\scriptsize $\pm$1.38} \\
ODIN~\cite{liang2017enhancing} & 80.05{\scriptsize $\pm$0.12} / 79.94{\scriptsize $\pm$0.95} & 58.27{\scriptsize $\pm$0.79} / 57.58{\scriptsize $\pm$2.51} & 82.94{\scriptsize $\pm$2.51} / 67.26{\scriptsize $\pm$1.71} \\
MDS~\cite{lee2018simple} & 68.10{\scriptsize $\pm$0.92} / 74.16{\scriptsize $\pm$0.66} & 74.68{\scriptsize $\pm$1.00} / 68.42{\scriptsize $\pm$1.47} & 71.91{\scriptsize $\pm$1.47} / 57.39{\scriptsize $\pm$1.32} \\
Gram~\cite{hendrycks2021unsolved} & 54.39{\scriptsize $\pm$0.31} / 70.65{\scriptsize $\pm$1.39} & 92.29{\scriptsize $\pm$0.19} / 70.25{\scriptsize $\pm$3.62} & 57.09{\scriptsize $\pm$3.62} / 55.18{\scriptsize $\pm$2.56} \\
EBO~\cite{liu2020energy} & \textbf{\textcolor{blue4}{81.14}}{\scriptsize $\pm$0.15} / \textbf{\textcolor{blue2}{80.59}}{\scriptsize $\pm$1.18} & 54.96{\scriptsize $\pm$0.89} / \textbf{\textcolor{blue2}{55.24}}{\scriptsize $\pm$2.77} & \textbf{\textcolor{blue2}{84.04}}{\scriptsize $\pm$2.77} / \textbf{\textcolor{blue3}{68.25}}{\scriptsize $\pm$1.89} \\
GradNorm~\cite{huang2021importance} & 76.04{\scriptsize $\pm$0.50} / 75.13{\scriptsize $\pm$2.35} & 78.87{\scriptsize $\pm$1.40} / 76.67{\scriptsize $\pm$2.97} & 75.64{\scriptsize $\pm$2.97} / 53.12{\scriptsize $\pm$2.90} \\
ReAct~\cite{sun2021react} & 80.71{\scriptsize $\pm$0.43} / \textbf{\textcolor{blue5}{81.45}}{\scriptsize $\pm$1.11} & 56.73{\scriptsize $\pm$2.29} / \textbf{\textcolor{blue5}{52.05}}{\scriptsize $\pm$2.70} & 83.31{\scriptsize $\pm$2.70} / \textbf{\textcolor{blue5}{69.41}}{\scriptsize $\pm$2.02} \\
MLS~\cite{basart2022scaling} & \textbf{\textcolor{blue5}{81.21}}{\scriptsize $\pm$0.12} / \textbf{\textcolor{blue1}{80.40}}{\scriptsize $\pm$1.07} & {54.88}{\scriptsize $\pm$0.83} / \textbf{\textcolor{blue1}{55.33}}{\scriptsize $\pm$2.71} & \textbf{\textcolor{blue3}{84.07}}{\scriptsize $\pm$2.71} / \textbf{\textcolor{blue2}{68.08}}{\scriptsize $\pm$1.83} \\
KLM~\cite{basart2022scaling} & 74.88{\scriptsize $\pm$0.46} / 73.32{\scriptsize $\pm$0.81} & 78.69{\scriptsize $\pm$4.95} / 82.01{\scriptsize $\pm$2.23} & 75.05{\scriptsize $\pm$2.23} / 50.89{\scriptsize $\pm$0.91} \\
VIM~\cite{wang2022vim} & 73.43{\scriptsize $\pm$0.07} / 79.14{\scriptsize $\pm$0.99} & 63.98{\scriptsize $\pm$0.42} / \textbf{\textcolor{blue4}{54.93}}{\scriptsize $\pm$0.84} & 77.84{\scriptsize $\pm$0.84} / 67.34{\scriptsize $\pm$0.81} \\
KNN~\cite{sun2022out} & 80.51{\scriptsize $\pm$0.12} / \textbf{\textcolor{blue4}{80.93}}{\scriptsize $\pm$0.28} & 59.31{\scriptsize $\pm$0.87} / 57.32{\scriptsize $\pm$1.65} & 81.73{\scriptsize $\pm$1.65} / \textbf{\textcolor{blue1}{67.56}}{\scriptsize $\pm$1.43} \\
DICE~\cite{sun2022dice} & 80.34{\scriptsize $\pm$0.26} / \textbf{\textcolor{blue3}{80.60}}{\scriptsize $\pm$1.40} & 55.97{\scriptsize $\pm$1.03} / \textbf{\textcolor{blue3}{55.03}}{\scriptsize $\pm$3.26} & 83.52{\scriptsize $\pm$3.26} / \textbf{\textcolor{blue4}{68.35}}{\scriptsize $\pm$2.18} \\
SIRC(MSP,$||z||_1$)~\cite{xia2022augmenting} & 80.66{\scriptsize $\pm$0.03} / 79.35{\scriptsize $\pm$1.06} & \textbf{\textcolor{blue2}{54.45}}{\scriptsize $\pm$0.53} / 56.78{\scriptsize $\pm$2.36} & 83.94{\scriptsize $\pm$2.36} / 66.68{\scriptsize $\pm$1.72} \\
SIRC(MSP,Res.)~\cite{xia2022augmenting} & 80.75{\scriptsize $\pm$0.05} / 79.11{\scriptsize $\pm$1.12} & \textbf{\textcolor{blue4}{54.41}}{\scriptsize $\pm$0.57} / 56.68{\scriptsize $\pm$2.35} & 83.97{\scriptsize $\pm$2.35} / 66.68{\scriptsize $\pm$1.75} \\
SIRC(-H,$||z||_1$)~\cite{xia2022augmenting} & \textbf{\textcolor{blue3}{81.10}}{\scriptsize $\pm$0.13} / 80.17{\scriptsize $\pm$1.12} & \textbf{\textcolor{blue1}{54.50}}{\scriptsize $\pm$0.56} / 56.37{\scriptsize $\pm$2.47} & \textbf{\textcolor{blue5}{84.20}}{\scriptsize $\pm$2.47} / 67.39{\scriptsize $\pm$1.79} \\
SIRC(-H,Res.)~\cite{xia2022augmenting} & \textbf{\textcolor{blue2}{81.07}}{\scriptsize $\pm$0.11} / 79.80{\scriptsize $\pm$1.15} & \textbf{\textcolor{blue3}{54.43}}{\scriptsize $\pm$0.65} / 56.27{\scriptsize $\pm$2.43} & \textbf{\textcolor{blue4}{84.18}}{\scriptsize $\pm$2.43} / 67.35{\scriptsize $\pm$1.81} \\
\midrule
\textbf{\textbf{ID Acc.}} & \multicolumn{1}{l}{} & \multicolumn{1}{c}{77.32} &  \\
\bottomrule
\end{tabular}
}
\caption{\textbf{Evaluation on ResNet-18 based on OpenOOD settings.} We report each metric on both Near- and Far-OOD tests. All experiments were run three times, and both the average and the standard deviation are presented. The top-5 methods for each metric are highlighted using a color gradient from \textbf{\textcolor{blue1}{light blue}} to \textbf{\textcolor{blue5}{dark blue}}.}
\label{tab:supp_openoodr18_color}
\end{table*}